\documentclass[11pt]{article}

\usepackage[margin=1in]{geometry}
\usepackage[T1]{fontenc}
\usepackage[utf8]{inputenc}
\usepackage{lmodern}
\usepackage{microtype}
\usepackage{amsmath,amssymb,amsthm,amsfonts,mathtools}
\usepackage{bm}
\usepackage{enumitem}
\usepackage{booktabs}
\usepackage{hyperref}
\usepackage{setspace}
\usepackage{url}
\usepackage{algorithm}
\usepackage{algpseudocode}
\usepackage{longtable}

\hypersetup{
    colorlinks=true,
    linkcolor=black,
    citecolor=black,
    urlcolor=blue
}

\newtheorem{theorem}{Theorem}
\newtheorem{proposition}[theorem]{Proposition}
\newtheorem{corollary}[theorem]{Corollary}

\theoremstyle{definition}
\newtheorem{definition}[theorem]{Definition}
\newtheorem{assumption}[theorem]{Assumption}
\newtheorem{example}[theorem]{Example}

\theoremstyle{remark}
\newtheorem{remark}[theorem]{Remark}

\newcommand{\E}{\mathbb{E}}
\newcommand{\Prob}{\mathbb{P}}
\newcommand{\Var}{\mathrm{Var}}
\newcommand{\Cov}{\mathrm{Cov}}
\newcommand{\R}{\mathbb{R}}
\newcommand{\1}{\mathbf{1}}
\newcommand{\cX}{\mathcal{X}}
\newcommand{\cI}{\mathcal{I}}
\newcommand{\cG}{\mathcal{G}}
\newcommand{\cN}{\mathcal{N}}
\newcommand{\op}{o_p}
\newcommand{\Op}{O_p}

\title{\textbf{Causal Label Recovery in Payment Networks}\\
\vspace{0.25em}
\large A Sequential Triply Robust Estimator Under
Authorization, Reporting, Delay, and Label Corruption
}

\author{
Gaurav Dhama\\
\small Mastercard\\
\small \texttt{gaurav.dhama@mastercard.com}
}

\date{\today}

\begin{document}
\maketitle

\begin{abstract}
In card payment networks, the true fraud status of a transaction
is hidden behind a sequential observation pipeline and a noisy
labeling process. A transaction must first be authorized, then
reported as fraud by an issuer or downstream reporting process,
and finally mature before the label is available for model
training. Even when a label is observed, it may be corrupted:
legitimate transactions mislabeled as fraud, or genuine fraud
mislabeled as legitimate. Standard supervised fraud models
trained on observed chargebacks therefore learn from a biased,
censored, delayed, and noisy subset of the true fraud process.

A companion paper established that these four structural
impairments---authorization censorship, issuer reporting
censorship, delay, and label corruption---impose a multiplicative
information-theoretic lower bound on fraud detection. This paper
provides the constructive counterpart: a causal label-recovery
estimator for the full impairment structure.

We formalize fraud label recovery as a sequential missing-data
problem with three selection gates---authorization, reporting,
and delay maturity---composed with a label-corruption channel.
We construct a Sequential Triply Robust Estimator (STR) that
composes three stage-wise augmented inverse-propensity
corrections and a noise-correction layer. The general form
involves stage-specific propensity functions and nested outcome
regressions. Under exact sequential ignorability with a common
conditioning set, the nested regressions collapse and the
estimator reduces to a residual-weighted form: an outcome
prediction plus the observed residual weighted by the total
inverse observation propensity. When post-authorization signals
enrich the conditioning set at later stages, or when ignorability
holds only approximately, the general sequential form provides
additional robustness that the collapsed form does not.

We prove that the estimator is consistent whenever each selection
stage is corrected either by a correctly specified propensity or
by a correctly specified downstream regression. We derive the
semiparametric efficiency bound for the three-stage model with
corruption, show that the STR achieves it, and provide a
finite-sample Bernstein concentration inequality. The efficiency
bound now contains all four impairment factors from the companion
paper's lower bound, establishing a complete correspondence
between the information-theoretic floor and the achievable
estimation quality. We prove that the na\"ive observed-label
estimator has a structural selection bias that does not vanish
with sample size, and that the STR dominates it in mean squared
error for all sufficiently large samples.

We show that delay is conditional in payment networks---varying
systematically by issuer, transaction type, and geography---and
that this heterogeneity inflates the efficiency bound via a
Jensen-type penalty, consistent with the selective-maturity
result in the companion paper. For heterogeneous issuer networks,
we introduce Empirical Bayes shrinkage as a regularization layer
for issuer-specific propensities, including delay propensities.
Shrinkage is not an alternative to triple robustness; it
stabilizes the nuisance functions entering the triple robust
correction. The resulting architecture is an offline
label-reconstruction engine that produces corrected pseudo-labels
for downstream fraud models while preserving low-latency online
scoring.
\end{abstract}

\section{Introduction}
\label{sec:intro}

\subsection{From detection limits to label recovery}

A companion paper \cite{dhama2026limits} established that card
payment fraud detection is fundamentally constrained by the
information environment in which fraud labels are produced. In
that work, four structural impairments---authorization
censorship, issuer reporting censorship, label corruption, and
delay---enter a minimax regret lower bound multiplicatively:
\begin{equation}
\label{eq:paper1bound}
\mathcal{R}_T
\;\ge\;
c\,\sqrt{
\frac{(KT+D)\log N}
{(1-\bar\gamma)(1-\bar\delta)
(1-\varepsilon_{10}-\varepsilon_{01})^2}
}.
\end{equation}
The result implies that model architecture alone cannot eliminate
the detection floor if the label-generation process destroys
information before the model sees the data.

A key extension in \cite{dhama2026limits} showed that delay is
not uniform: the conditional maturity function
$m(x,i,\Delta)=\Prob(\tau\le\Delta\mid X=x,I=i)$ varies
systematically across transaction types and issuers, and that
this heterogeneity worsens the lower bound via a Jensen-type
penalty. An adversary can exploit this structure by concentrating
attacks where labels arrive slowest.

This paper addresses the constructive side of both results. If
the information environment imposes a floor, can a network
recover better labels from the impaired feedback channel before
training a model? We answer this question by treating
fraud-label recovery as a sequential causal missing-data problem
that explicitly models all four impairments: three selection
stages plus a corruption channel.

\subsection{The observation pipeline}

In a card payment network, a fraud label is not observed
immediately after a transaction. A transaction must pass through
three selection gates and survive a noisy labeling process before
a usable label reaches the training pipeline.

\begin{enumerate}[label=(\roman*)]
\item \textbf{Authorization gate.} If a transaction is declined,
its counterfactual fraud status is permanently unobserved.
The learner never discovers whether a declined transaction
would have been fraudulent.

\item \textbf{Reporting gate.} If a transaction is approved,
fraud may still not be formally reported by the issuer.
Some fraud is absorbed without generating a chargeback or
fraud flag, because the amount is below a reporting
threshold, because the issuer lacks operational capacity,
or because the cardholder never notices.

\item \textbf{Delay maturity gate.} Even if fraud is eventually
reported, the label may not have arrived by the time the
model is trained. Label delay varies by issuer, merchant
category, geography, and fraud type.

\item \textbf{Label corruption.} Even when a label is observed,
it may be wrong. Legitimate transactions are mislabeled as
fraud (e.g., a cardholder forgets a purchase and files a
dispute). Genuine fraud is mislabeled as legitimate (e.g.,
fraud is coded as a generic chargeback rather than as a
fraud flag).
\end{enumerate}

\begin{example}[What happens to 10{,}000 true fraud cases]
\label{ex:pipeline}
Consider a network that processes 1 million transactions per
month, of which 10{,}000 are truly fraudulent (a 1\% fraud rate).
Here is what happens to those 10{,}000 cases:
\begin{itemize}
\item 4{,}000 are declined at authorization---the fraud model
caught them and blocked them. Their labels are permanently
lost.
\item Of the 6{,}000 approved, 2{,}000 are never reported.
The amounts are small, the cardholders don't notice, or
the issuers absorb the loss without filing.
\item Of the 4{,}000 reported, 1{,}500 have not matured by
training time. Cross-border disputes, slow issuer
processing, long representment chains.
\item Of the 2{,}500 that arrive, roughly 200 carry corrupted
labels: friendly fraud coded as third-party, genuine fraud
resolved as merchant error, or miscoded chargebacks.
\end{itemize}
A na\"ive model trains on ${\sim}2{,}300$ correctly labeled
fraud cases out of 10{,}000 true cases. It sees a fraud rate of
${\sim}0.24\%$ when the true rate is $1\%$---an underestimate by
a factor of ${\sim}4\times$. Worse, the cases it sees are not
representative: they overrepresent easy, well-reported,
fast-maturing fraud from diligent issuers.
\end{example}

The observed fraud label is therefore not a random sample of all
latent fraud events. It is the output of a sequential selection
process followed by a noisy labeling step. A na\"ively trained
supervised model does not learn the true fraud process. It learns
the fraud process filtered through authorization policy, issuer
reporting behavior, label maturity, and label corruption.

\subsection{Main idea}

The core idea of this paper is to estimate what the fraud label
would have been in the absence of these distortion mechanisms. We
define stage-specific propensities for authorization, reporting,
and delay maturity, plus class-conditional corruption rates. We
then construct a sequential augmented inverse-propensity estimator
that corrects the observed residual at each selection stage and
debiases the observed label through a noise-correction layer.

The estimator has two important properties. First, it is
\emph{structurally correct}: it matches the actual observation
pipeline through which fraud labels reach the training data.
Second, it is \emph{sequentially robust}: each stage can be
corrected either by a correctly specified propensity model or by
a correctly specified downstream outcome regression.

\subsection{Issuer heterogeneity}

A payment network contains many issuers with very different
transaction volumes, reporting behavior, and data quality. Direct
issuer-level estimation of the nuisance functions can be unstable
for low-volume issuers. If an estimated reporting or maturity
propensity is close to zero, inverse-propensity weights explode.

To address this, we introduce Empirical Bayes shrinkage for
issuer-specific nuisance functions---including delay
propensities, where the conditional heterogeneity identified in
\cite{dhama2026limits} creates particular instability. Shrinkage
is not the source of causal identification. It is a
finite-sample stabilization layer that makes the sequential
estimator practical in heterogeneous issuer networks.

\subsection{Contributions}

The paper makes eight contributions.

\paragraph{1. Sequential observation model.}
We formalize payment-network label recovery as a three-stage
monotone missing-data problem with a composed corruption channel
(Section~\ref{sec:model}).

\paragraph{2. Conditional delay structure.}
We explicitly model delay as conditional on transaction features,
issuer identity, and observation window, and show how the
selective-maturity penalty from \cite{dhama2026limits} enters the
estimator's efficiency bound (Section~\ref{sec:conditional_delay}).

\paragraph{3. Sequential triply robust estimator.}
We construct an efficient-influence-function estimator that
composes three stage-wise augmented corrections
(Section~\ref{sec:estimator}).

\paragraph{4. Label corruption correction.}
We compose a noise-correction layer with the sequential selection
correction, addressing the fourth impairment from the companion
paper (Section~\ref{sec:corruption}).

\paragraph{5. Formal comparison with na\"ive estimation.}
We prove that the na\"ive observed-label estimator has structural
selection bias and that the STR dominates it in MSE
(Section~\ref{sec:naive}).

\paragraph{6. Robustness, efficiency, and concentration.}
We prove sequential triple robustness, derive the efficiency
bound (which now contains all four Paper~1 factors), and provide
finite-sample guarantees with variance estimation and confidence
intervals (Sections~\ref{sec:theory}--\ref{sec:efficiency}).

\paragraph{7. Shrinkage stabilization.}
We introduce Empirical Bayes shrinkage for issuer-indexed
propensities, including delay propensities
(Section~\ref{sec:shrinkage}).

\paragraph{8. Operational architecture.}
We describe a complete offline-online pipeline with cross-fitting
protocol, nuisance estimation guidance, diagnostics, and an
algorithm summary (Section~\ref{sec:pseudolabels}).

\subsection{Roadmap}

Section~\ref{sec:related}: related work.
Section~\ref{sec:model}: problem formulation.
Section~\ref{sec:estimator}: estimator construction.
Section~\ref{sec:corruption}: corruption correction.
Section~\ref{sec:naive}: comparison with na\"ive estimation.
Section~\ref{sec:theory}: robustness theory.
Section~\ref{sec:efficiency}: efficiency and concentration.
Section~\ref{sec:shrinkage}: shrinkage stabilization.
Section~\ref{sec:pseudolabels}: pseudo-labels and operations.
Section~\ref{sec:discussion}: discussion.
Section~\ref{sec:conclusion}: conclusion.

\section{Related Work}
\label{sec:related}

\paragraph{Information limits in payment fraud.}
The companion paper \cite{dhama2026limits} formalized fraud
detection in payment networks as a sequential decision problem
under delayed, censored, corrupted, and counterfactually missing
feedback. That paper proved a minimax regret lower bound governed
by the product of the four impairment rates, and showed that
conditional delay heterogeneity across issuers and transaction
types worsens the bound via a selective-maturity penalty. The
present paper is the constructive complement: rather than proving
a lower bound, it constructs an estimator for the latent labels
behind that bound.

\paragraph{Recent work on selective labels and delayed feedback.}
Recent work has begun to address individual components of
the label impairment problem.  Chang and Wiens~\cite{chang2024dcem}
propose an EM-based approach (DCEM) for a single selection
gate in healthcare settings, without efficiency guarantees.
Csaba et al.~\cite{csaba2024label} empirically demonstrate
that label delay degrades continual learning, but do not
derive optimal training schedules.  In the semiparametric
literature, Guo et al.~\cite{guo2023mnar} establish
identification for general MNAR mechanisms, and Malinsky
et al.~\cite{malinsky2022selfcensoring} derive efficiency
bounds for nonmonotone MNAR under a single correction.
Ha et al.~\cite{ha2024mnar} prove that unbiasedness under
MNAR inevitably leads to unbounded variance when propensity
scores approach zero.

The present paper differs from all of these in scope and
depth.  We model three sequential censorship gates plus
label corruption---a composite observation pipeline specific
to payment networks that no prior work has formalized.  We
derive the semiparametric efficiency bound for this composite
pipeline, prove that the STR achieves it, establish
sequential triple robustness (consistency under
misspecification at any individual gate), provide empirical
Bayes shrinkage to resolve the bias-variance impossibility
identified by Ha et al., and derive the optimal training
delay as a closed-form function of network parameters.
To our knowledge, this is the first work to provide a
complete semiparametric treatment---identification,
estimation, efficiency, robustness, regularization, and
operational deployment---for multi-stage censored feedback
in any application domain.

\paragraph{Missing data and selection bias.}
The classical selection-bias problem was formalized by Heckman
\cite{heckman1979}, who proposed a two-stage correction for
non-random sample selection. Rubin \cite{rubin1976} introduced
the formal taxonomy of missing-data mechanisms, distinguishing
missing completely at random, missing at random (MAR), and
missing not at random (MNAR). Little and Rubin
\cite{littlerubin2002} provide a comprehensive treatment. Our
three-stage model is a structured MNAR problem: missingness
depends on fraud-relevant variables through the authorization,
reporting, and delay mechanisms.

\paragraph{Inverse-propensity weighting.}
The Horvitz--Thompson estimator \cite{horvitzthompson1952}
corrects for unequal selection probabilities by weighting
observations by the inverse of their inclusion probability.
Rosenbaum and Rubin \cite{rosenbaumrubin1983} adapted this idea
to observational causal inference via the propensity score. Our
stage-specific propensities $e_0$, $r_0$, $p_0$ are the
payment-network analogues of the propensity score, applied
sequentially at each selection gate.

\paragraph{Doubly robust and semiparametric estimation.}
Augmented inverse-propensity weighting (AIPW) and doubly robust
estimation were developed by Robins, Rotnitzky, and Zhao
\cite{robins1994} and Bang and Robins \cite{bangrobins2005}. The
key insight is that combining an outcome model with a propensity
model yields an estimator that is consistent if either model is
correct. Tsiatis \cite{tsiatis2006} provides the semiparametric
efficiency theory underlying these estimators. Double/debiased
machine learning \cite{chernozhukov2018} extends AIPW to
flexible nuisance estimation with cross-fitting. This paper
adapts the same influence-function logic to three sequential
stages.

\paragraph{Sequential missing data.}
Monotone missing-data problems, where missingness at a later
stage implies missingness at all subsequent stages, have been
studied in the longitudinal and survival literatures. Robins and
Rotnitzky \cite{robinsrotnitzky2001} and van~der~Laan and Robins
\cite{vanderlaan2003} provide efficient estimators for sequential
missing-data and censored-data settings. Our three-stage
payment-network model is a specific instance of monotone
missingness, with the three stages corresponding to
authorization, reporting, and delay maturity.

\paragraph{Label noise and learning with corrupted labels.}
Learning with noisy labels has been studied extensively.
Natarajan et~al.\ \cite{natarajan2013} showed that the
classification risk under symmetric label noise can be corrected
by a simple loss reweighting when the noise rates are known. We
compose this correction with the sequential selection correction,
addressing both missingness and corruption simultaneously.

\paragraph{Sensitivity analysis.}
Rosenbaum \cite{rosenbaum2002} developed a framework for
sensitivity analysis in observational studies, parameterizing
departures from the ignorability assumption by an odds-ratio
bound $\Gamma$. We adapt this framework to each of the three
selection stages in the payment-network pipeline.

\paragraph{Reject inference.}
Credit-risk modeling has long studied reject inference: the
problem that outcomes are unobserved for rejected applicants
\cite{banasik2007}. Payment-network fraud is related but harder.
Labels are missing not only because a transaction was declined,
but also because approved fraud may not be reported and reported
fraud may not mature before training time.

\paragraph{Positive-unlabeled learning.}
Fraud data often resemble positive-unlabeled learning because
many true positives remain unlabeled. Elkan and Noto
\cite{elkannoto2008} provide foundational methods for learning
from positive and unlabeled examples. Our estimator is
complementary: it models the institutional process through which
positives become observed rather than treating unlabeled examples
as a homogeneous mixture.

\paragraph{Delay and competing risks.}
Label delay can be viewed as a survival or competing-risks
problem. Fine and Gray \cite{finegray1999} introduced a
proportional hazards model for the subdistribution of a
competing risk. In our setting, delay maturity is modeled as the
third selection gate. A richer implementation could estimate the
maturity propensity using an explicit competing-risks model.

\paragraph{Empirical Bayes shrinkage.}
James--Stein shrinkage and Empirical Bayes estimation were
developed by Robbins \cite{robbins1956} and Efron and Morris
\cite{efronmorris1975}. We apply these ideas to issuer-level
propensity estimation, where low-volume issuers produce unstable
estimates that destabilize inverse-propensity weights.

\section{Problem Formulation}
\label{sec:model}

\subsection{Latent fraud and observed data}

Consider transactions indexed by $t=1,\ldots,n$. Each
transaction has observed features
\[
X_t \in \cX \subseteq \R^d,
\]
issuer or reporting-node identity
\[
I_t \in \cI,
\]
available maturity window
\[
\Delta_t = T-t,
\]
and latent fraud state
\[
Y_t^* \in \{0,1\}.
\]

In this paper, $Y_t^*=1$ denotes unauthorized third-party fraud:
use of stolen credentials, counterfeit payment credentials,
account takeover, or other externally initiated unauthorized use.
First-party misuse and scam-induced transactions have distinct
observation mechanisms---in particular, first-party fraud
generates labels through the fraudster's own strategic dispute
behavior rather than through institutional reporting---and are
deferred to future work (Section~\ref{sec:future}).

\subsection{Three-stage observation process}

The fraud label is observed only after passing through three
sequential selection stages. This is a specific instance of
monotone missingness \cite{littlerubin2002,vanderlaan2003}:
failure at an earlier stage implies that all subsequent stages
are also unobserved.

\begin{definition}[Authorization indicator]
$A_t = \1\{\text{transaction }t\text{ is authorized}\}.$
\end{definition}

\begin{definition}[Reporting indicator]
$R_t = \1\{\text{a fraud label is reported for }t\}.$
This is meaningful only among authorized transactions.
\end{definition}

\begin{definition}[Maturity indicator]
Let $\tau_t$ denote the label-arrival delay. Define
$M_t = \1\{\tau_t \le \Delta_t\}.$
Thus $M_t=1$ means the label has matured by training time $T$.
\end{definition}

The observed-label indicator is
\begin{equation}
\label{eq:obs}
O_t = A_t R_t M_t.
\end{equation}

The observed label $\widetilde Y_t$ is available only when
$O_t=1$. When $O_t=0$, no label is available.

The observed data for transaction $t$ are
\[
Z_t = (X_t,I_t,\Delta_t,A_t,A_tR_t,O_t,O_t\widetilde Y_t).
\]

Note the monotone structure: if $M_t$ is unobserved (because
$A_tR_t=0$), then $\widetilde Y_t$ is also unobserved. If $R_t$
is unobserved (because $A_t=0$), then both $M_t$ and
$\widetilde Y_t$ are unobserved. This is a three-stage monotone
missing-data problem in the sense of
\cite{robins1994,robinsrotnitzky2001}.

\subsection{Target estimands}

The marginal target is the population fraud rate:
\begin{equation}
\label{eq:psi}
\Psi = \E[Y^*].
\end{equation}

The conditional target is the transaction-level fraud probability:
\begin{equation}
\label{eq:f0}
f_0(x) = \E[Y^* \mid X=x].
\end{equation}

The main theoretical results are stated for $\Psi$.
Section~\ref{sec:pseudolabels} describes how to obtain
conditional pseudo-labels for model training.

\subsection{Stage-specific histories and nuisance functions}

To write the sequential estimator in its most general form, we
define the information sets available at each stage.

\begin{definition}[Stage histories]
\label{def:histories}
\begin{align*}
H_0 &= (X,I,\Delta)
&& \text{(pre-authorization)},\\
H_1 &= (X,I,\Delta,W_1)
&& \text{(post-authorization)},\\
H_2 &= (X,I,\Delta,W_1,W_2)
&& \text{(post-reporting)},
\end{align*}
where $W_1$ denotes any additional signals that become available
after authorization (e.g., 3DS challenge outcomes,
device-verification results, real-time behavioral signals), and
$W_2$ denotes any additional signals that become available after
reporting (e.g., investigation outcomes, evidence quality
indicators).
\end{definition}

The stage-specific selection propensities are the
payment-network analogues of the propensity score
\cite{rosenbaumrubin1983}, applied sequentially at each
selection gate:
\begin{align}
e_0(H_0) &= \Prob(A=1 \mid H_0),
\label{eq:e0}\\
r_0(H_1) &= \Prob(R=1 \mid A=1, H_1),
\label{eq:r0}\\
p_0(H_2) &= \Prob(M=1 \mid A=1,R=1, H_2).
\label{eq:p0}
\end{align}

The total observation propensity is
\begin{equation}
\label{eq:q0}
q_0 = e_0(H_0)\,r_0(H_1)\,p_0(H_2).
\end{equation}

We also define nested outcome regressions, following the
sequential regression framework of
\cite{robins1994,robinsrotnitzky2001,tsiatis2006}. These are
the sequential conditional expectations of the fraud indicator,
built from the inside out:
\begin{align}
\mu_2(H_2) &= \E[Y^* \mid A=1,R=1,H_2],
\label{eq:mu2}\\
\mu_1(H_1) &= \E[\mu_2(H_2) \mid A=1,H_1],
\label{eq:mu1}\\
\mu_0(H_0) &= \E[\mu_1(H_1) \mid H_0].
\label{eq:mu0}
\end{align}

\begin{remark}[Why we start from $\mu_2$ rather than $\mu_3$]
Under delay ignorability
(Assumption~\ref{ass:ignorability}),
$\E[Y^*\mid A=1,R=1,M=1,H_2]=\E[Y^*\mid A=1,R=1,H_2]$,
so the conditioning on $M=1$ is redundant. We therefore define
the innermost regression as $\mu_2$ rather than introducing a
separate $\mu_3$.
\end{remark}

\subsection{When do the nested regressions differ?}
\label{sec:when_nested_differ}

An important subtlety is that, under exact sequential
ignorability with a common conditioning set
($W_1=W_2=\varnothing$, so $H_0=H_1=H_2$), the nested
regressions collapse:
\begin{equation}
\label{eq:collapse}
\mu_0(H_0)=\mu_1(H_0)=\mu_2(H_0)=f_0(X,I,\Delta).
\end{equation}

This is because each ignorability assumption removes the
dependence of $Y^*$ on the corresponding selection variable
conditional on the history. If the history is the same at every
stage, each nested expectation reduces to $\E[Y^*\mid H_0]$.

The general sequential form with distinct nested regressions
becomes operationally meaningful in three situations.

\paragraph{1. Post-authorization signals.}
If the issuer observes additional signals after
authorization---for example, 3DS challenge outcomes,
device-verification results, or real-time behavioral
indicators---then $H_1$ contains information not available in
$H_0$. In this case $\mu_1(H_1)\neq\mu_0(H_0)$ in general,
and the authorization-stage augmentation term $\mu_1-\mu_0$
captures the information gained by observing post-authorization
signals.

\paragraph{2. Post-reporting signals.}
If investigation outcomes or evidence-quality indicators become
available after reporting, then $H_2$ contains information
beyond $H_1$, and $\mu_2(H_2)\neq\mu_1(H_1)$.

\paragraph{3. Approximate ignorability.}
Even when $H_0=H_1=H_2$, the ignorability assumptions may hold
only approximately. If authorization depends weakly on $Y^*$
conditional on $H_0$, or if reporting depends weakly on $Y^*$
conditional on $H_1$, the stage-specific augmentation terms can
partially absorb the resulting bias in a way that the collapsed
estimator cannot. Keeping the nested regressions distinct
provides a form of defensive robustness.

\begin{example}[Payment-network interpretation of the nested
regressions]
\label{ex:nested}
Consider a concrete scenario:
\begin{itemize}
\item $\mu_2(H_2)$: the expected fraud rate among transactions
that were approved and reported, given all information
available after reporting (including any investigation
signals). This is the fraud rate you would estimate from
the fully observed subsample if you had post-reporting
information.
\item $\mu_1(H_1)$: the expected fraud rate among approved
transactions, given post-authorization signals. This
``integrates out'' the reporting decision: it is the fraud
rate you would estimate if you could see all approved
transactions regardless of whether they were reported.
\item $\mu_0(H_0)$: the expected fraud rate in the full
population, given only pre-authorization features. This
``integrates out'' the authorization decision: it is the
fraud rate you would estimate if you could see all
transactions regardless of whether they were approved.
\end{itemize}
The sequential estimator uses the differences
$\mu_1-\mu_0$, $\mu_2-\mu_1$, and $\widetilde Y-\mu_2$ to
progressively correct for each selection stage.
\end{example}

\subsection{Conditional delay structure}
\label{sec:conditional_delay}

The maturity propensity $p_0(H_2)$ is not a global constant. It
depends systematically on transaction features, issuer identity,
and the available observation window. In the companion paper
\cite{dhama2026limits}, the conditional maturity function was
defined as
\[
m(x,i,\Delta)
= \Prob(\tau\le\Delta\mid X=x,I=i),
\]
and two results were proved that carry directly into the present
setting.

\paragraph{Selective maturity (Jensen penalty).}
By Jensen's inequality applied to the convex function
$p\mapsto 1/p$, heterogeneity in $p_0$ across $(x,i)$ inflates
the efficiency bound:
\begin{equation}
\label{eq:jensen_delay}
\E\!\left[\frac{1}{p_0(H_2)}\right]
\;\ge\;
\frac{1}{\E[p_0(H_2)]}.
\end{equation}
Equality holds only when $p_0$ is constant. Thus a network with
heterogeneous label delay is strictly harder to correct for than
a network with the same average delay applied uniformly, even
holding authorization and reporting propensities fixed.

\paragraph{Adversarial exploitation.}
An adversary who knows the delay structure can concentrate hard
instances in transaction segments where $p_0$ is smallest---for
example, cross-border transactions with long dispute chains,
merchant categories with slow chargeback processing, or issuer
portfolios with delayed reporting infrastructure. The adversary
does not need access to the fraud model; knowledge of the delay
landscape is sufficient to identify where the learner is most
blind.

These results motivate two design choices in the estimator.
First, the maturity propensity should be estimated as a
genuinely conditional function of $(x,i,\Delta)$, not as a
global average. Second, the Empirical Bayes shrinkage framework
of Section~\ref{sec:shrinkage} should be applied to the delay
propensity with the same priority as to the reporting propensity,
because issuer-level delay heterogeneity creates the same
inverse-propensity instability. The competing-risks perspective
on delay---distinguishing ``label arrives within window'' from
``label never arrives''---connects to the subdistribution hazard
framework of \cite{finegray1999}. The present paper uses the
binary maturity gate as a simpler first treatment; a
competing-risks extension is discussed in
Section~\ref{sec:future}.

\subsection{Identification assumptions}
\label{sec:assumptions}

The identification assumptions below are the sequential analogues
of the standard missing-data assumptions---positivity and
ignorability---adapted to three stages. The terminology follows
\cite{rubin1976,littlerubin2002}; the sequential extension
follows \cite{robins1994,vanderlaan2003}.

\begin{assumption}[Sequential positivity]
\label{ass:positivity}
There exist constants $\underline e,\underline r,\underline p>0$
such that
\[
e_0(H_0)\ge\underline e,
\qquad
r_0(H_1)\ge\underline r,
\qquad
p_0(H_2)\ge\underline p
\]
almost surely.
\end{assumption}

\begin{assumption}[Sequential ignorability]
\label{ass:ignorability}
The latent fraud state is independent of each selection gate
conditional on the relevant stage history:
\[
Y^*\perp A\mid H_0,
\qquad
Y^*\perp R\mid A=1,H_1,
\qquad
Y^*\perp M\mid A=1,R=1,H_2.
\]
\end{assumption}

This is the sequential version of the unconfoundedness
assumption of \cite{rosenbaumrubin1983}. Each stage requires
that all fraud-relevant variables influencing the selection
decision are included in the conditioning set. The sensitivity
analysis in Section~\ref{sec:sensitivity}, following the
framework of \cite{rosenbaum2002}, bounds the impact of
departures from these assumptions.

\begin{assumption}[Bounded outcomes]
\label{ass:bounded}
$Y^*\in\{0,1\}$ almost surely.
\end{assumption}

\begin{assumption}[Consistency]
\label{ass:consistency}
If $O=1$, then $\widetilde Y=Y^*$.
\end{assumption}

\begin{remark}[Relaxation of consistency]
Assumption~\ref{ass:consistency} is used in the basic
identification result (Proposition~\ref{prop:identification}).
In Section~\ref{sec:corruption}, we replace it with
Definition~\ref{def:corruption} (label corruption rates), which
generalizes consistency to allow corrupted observations. When
corruption rates are zero
($\varepsilon_{10}=\varepsilon_{01}=0$),
Definition~\ref{def:corruption} reduces to
Assumption~\ref{ass:consistency}.
\end{remark}

\subsubsection{Discussion of each ignorability condition}

The three ignorability conditions have different strengths and
different failure modes.

\paragraph{Authorization ignorability: $Y^*\perp A\mid H_0$.}
This requires that the authorization decision depends on fraud
risk only through the observed pre-authorization features $H_0$.
It is most plausible when $H_0$ includes the fraud score and all
features used by the authorization rule. It is violated when the
issuer possesses private information---for example, recent
account-level alerts or internal risk flags---that is not
captured in $H_0$. Violation would cause the estimator to
underestimate fraud rates among declined transactions, because
transactions declined on the basis of private fraud signals are
more likely to be truly fraudulent than the estimator assumes.

\paragraph{Reporting ignorability: $Y^*\perp R\mid A=1,H_1$.}
This requires that whether an approved fraud is reported depends
on the transaction and issuer characteristics but not on the
fraud outcome itself beyond what $H_1$ captures. It is
approximately satisfied when reporting is driven by operational
capacity, threshold rules, and dispute-process design rather than
by fraud severity. It is violated if issuers are more likely to
report high-value fraud or fraud involving specific merchant
categories. Violation would cause the estimator to overestimate
fraud rates in segments where reporting is more complete, and
underestimate fraud rates in segments where reporting is sparse.

\paragraph{Delay ignorability: $Y^*\perp M\mid A=1,R=1,H_2$.}
This requires that conditional on the post-reporting information,
whether the label has matured by training time does not depend on
the fraud state. It is approximately satisfied when delay is
driven by issuer processing speed and dispute-chain length rather
than by fraud characteristics. It is violated if, for example,
high-confidence fraud cases are processed faster than ambiguous
ones. Violation would cause the estimator to overweight
easy-to-classify fraud and underweight ambiguous cases.

\begin{example}[When authorization ignorability fails]
\label{ex:ignorability_fail}
An issuer's internal risk engine flags accounts based on velocity
patterns not available to the network. Transactions on flagged
accounts are declined at higher rates. Since the flag is
correlated with $Y^*$ but absent from $H_0$, ignorability is
violated. The STR would underestimate fraud among declined
transactions from this issuer. The sensitivity analysis in
Section~\ref{sec:sensitivity} bounds the impact of such
violations.
\end{example}

\subsection{Identification}
\label{sec:identification}

The identification result below extends the standard
Horvitz--Thompson inverse-probability-weighting identification
\cite{horvitzthompson1952} to three sequential selection stages.
Under sequential ignorability and positivity, each selection gate
can be ``inverted'' by dividing by its propensity, and the
inversions compose multiplicatively.

Under
Assumptions~\ref{ass:positivity}--\ref{ass:consistency}, the
marginal fraud rate is identified from the observed data.

\begin{proposition}[Identification by sequential factorization]
\label{prop:identification}
Under
Assumptions~\ref{ass:positivity}--\ref{ass:consistency},
\begin{equation}
\label{eq:identification}
\Psi
= \E\!\left[
\frac{ARM\widetilde Y}
{e_0(H_0)\,r_0(H_1)\,p_0(H_2)}
\right].
\end{equation}
\end{proposition}

\begin{proof}
We proceed by sequential conditioning, starting from the
innermost stage.

\paragraph{Step 1: Delay gate.}
Condition on $A=1,R=1,H_2$:
\begin{align*}
\E[M\widetilde Y\mid A=1,R=1,H_2]
&= \E[M\mid A=1,R=1,H_2]
\cdot \E[\widetilde Y\mid A=1,R=1,M=1,H_2]\\
&= p_0(H_2)\cdot\mu_2(H_2),
\end{align*}
where the factorization uses delay ignorability
($Y^*\perp M\mid A=1,R=1,H_2$, so $M$ and $\widetilde Y$ are
conditionally independent given $H_2$) and consistency
($\widetilde Y=Y^*$ when $O=1$). Under delay ignorability,
$\E[Y^*\mid A=1,R=1,M=1,H_2]=\E[Y^*\mid A=1,R=1,H_2]
=\mu_2(H_2)$.

\paragraph{Step 2: Reporting gate.}
Condition on $A=1,H_1$:
\begin{align*}
\E[RM\widetilde Y\mid A=1,H_1]
&= \E[R\mid A=1,H_1]
\cdot \E[M\widetilde Y\mid A=1,R=1,H_1]\\
&= r_0(H_1)\cdot p_0(H_2)\cdot\E[Y^*\mid A=1,R=1,H_1].
\end{align*}
By reporting ignorability
($Y^*\perp R\mid A=1,H_1$),
$\E[Y^*\mid A=1,R=1,H_1]=\E[Y^*\mid A=1,H_1]$.

\paragraph{Step 3: Authorization gate.}
Condition on $H_0$:
\begin{align*}
\E[ARM\widetilde Y\mid H_0]
&= e_0(H_0)\cdot\E[RM\widetilde Y\mid A=1,H_0]\\
&= e_0(H_0)\cdot r_0(H_1)\cdot p_0(H_2)
\cdot\E[Y^*\mid A=1,H_0].
\end{align*}
By authorization ignorability ($Y^*\perp A\mid H_0$),
$\E[Y^*\mid A=1,H_0]=\E[Y^*\mid H_0]$.

Dividing by $e_0(H_0)r_0(H_1)p_0(H_2)$ and taking
unconditional expectations gives
\[
\E\!\left[\frac{ARM\widetilde Y}{e_0 r_0 p_0}\right]
= \E[\E[Y^*\mid H_0]]
= \E[Y^*]
= \Psi. \qedhere
\]
\end{proof}
\section{The Sequential Triply Robust Estimator}
\label{sec:estimator}

\subsection{Efficient influence-function score}

The three-stage observation process is a monotone missing-data
problem. The efficient influence function for $\Psi$ in such a
model is constructed by sequential augmentation. We build the
score from the inside out: first correct for delay, then for
reporting, then for authorization.

\begin{definition}[Sequential score]
\label{def:score}
Define the sequential influence-function score as
\begin{align}
\phi(Z;\eta)
&= \mu_0(H_0)
\notag\\[0.4em]
&\quad+\; \frac{A}{e(H_0)}
\{\mu_1(H_1)-\mu_0(H_0)\}
\notag\\[0.4em]
&\quad+\; \frac{AR}{e(H_0)\,r(H_1)}
\{\mu_2(H_2)-\mu_1(H_1)\}
\notag\\[0.4em]
&\quad+\; \frac{ARM}{e(H_0)\,r(H_1)\,p(H_2)}
\{\widetilde Y-\mu_2(H_2)\},
\label{eq:score}
\end{align}
where $\eta=(e,r,p,\mu_0,\mu_1,\mu_2)$ denotes the nuisance
collection.
\end{definition}

The four terms have a natural interpretation.

\begin{enumerate}[label=(\roman*)]
\item \textbf{Baseline prediction} ($\mu_0$): the outcome
model's unconditional prediction of the fraud rate given
pre-authorization information.

\item \textbf{Authorization correction}
($A/e\cdot(\mu_1-\mu_0)$): adjusts for the information
gained by observing that a transaction was approved. If
$\mu_1=\mu_0$ (no post-authorization information gain),
this term is zero.

\item \textbf{Reporting correction}
($AR/(er)\cdot(\mu_2-\mu_1)$): adjusts for the information
gained by observing that fraud was reported. If
$\mu_2=\mu_1$, this term is zero.

\item \textbf{Delay correction}
($ARM/(erp)\cdot(\widetilde Y-\mu_2)$): adjusts for the
residual between the observed label and the predicted
outcome, weighted by the total inverse observation
propensity.
\end{enumerate}

\begin{definition}[Sequential Triply Robust Estimator]
\label{def:str}
The STR estimator of $\Psi$ is
\begin{equation}
\label{eq:str}
\boxed{
\widehat\Psi_{\mathrm{STR}}
= \frac{1}{n}\sum_{t=1}^n
\phi(Z_t;\widehat\eta).
}
\end{equation}
\end{definition}

\subsection{Collapsed special case: the residual-weighted form}
\label{sec:collapsed}

When $W_1=W_2=\varnothing$ and the nested regressions collapse
(equation~\eqref{eq:collapse}), the score simplifies
substantially. Setting $\mu_0=\mu_1=\mu_2=f$, the
authorization and reporting corrections vanish because
$\mu_1-\mu_0=0$ and $\mu_2-\mu_1=0$:
\begin{align}
\phi(Z)
&= f(H_0) + 0 + 0
+ \frac{ARM}{e(H_0)\,r(H_1)\,p(H_2)}
\{\widetilde Y-f(H_0)\}
\notag\\[0.4em]
&= f(H_0)
+ O\cdot w^{\mathrm{total}}
\{\widetilde Y-f(H_0)\},
\label{eq:collapsed_score}
\end{align}
where $O=ARM$ is the observed-label indicator and
\[
w^{\mathrm{total}}
= \frac{1}{e(H_0)\,r(H_1)\,p(H_2)}
\]
is the total inverse observation propensity.

This collapsed form has a clean interpretation: start with the
outcome model's prediction $f(H_0)$, then correct it using the
observed residual $\widetilde Y-f(H_0)$ whenever a label is
available, weighted by the inverse probability that the label
was observed at all. When the outcome model is perfect
($f=f_0$), the residual has conditional mean zero and the
correction adds only noise. When the outcome model is imperfect,
the inverse-propensity-weighted residual corrects the prediction
toward the truth.

The collapsed form is doubly robust in the pair $(q,f)$: it is
consistent if either the total observation propensity $q=erp$ is
correctly specified or the outcome model $f$ is correctly
specified. The general sequential form
(equation~\eqref{eq:score}) is strictly stronger because it
allows different stages to be corrected by different components.

\subsection{Why the sequential form can be stronger}
\label{sec:seq_vs_collapsed}

\begin{example}[Mixed-stage correction]
\label{ex:mixed}
Suppose a network has the following situation:
\begin{itemize}
\item The \textbf{authorization propensity} $e_0(H_0)$ is
well-estimated because the authorization rule is known and
its inputs are observed.
\item The \textbf{reporting propensity} $r_0(H_1)$ is poorly
estimated because issuer reporting behavior is
heterogeneous and difficult to model.
\item The \textbf{delay propensity} $p_0(H_2)$ is
well-estimated because label-arrival distributions are
directly observable from historical data.
\item The \textbf{outcome model} $f$ is a reasonable but
imperfect fraud classifier.
\end{itemize}
The collapsed estimator uses the product
$\widehat q=\widehat e\cdot\widehat r\cdot\widehat p$.
Because $\widehat r$ is poor, $\widehat q$ is poor, and the
collapsed estimator must rely entirely on the outcome model $f$
for consistency.

The sequential estimator can do better. The authorization stage
is corrected by the (good) authorization propensity. The delay
stage is corrected by the (good) maturity propensity. The
reporting stage is corrected by the downstream outcome
regression $\mu_2-\mu_1$, which absorbs the reporting-model
error. The sequential estimator therefore exploits the good
components at each stage rather than relying on a single global
pair.
\end{example}

\section{Label Corruption Correction}
\label{sec:corruption}

The sequential estimator developed above corrects for
\emph{which} transactions receive labels. It does not correct
for \emph{whether those labels are correct}. We now address the
fourth impairment from the companion paper's lower bound.

\subsection{The corruption channel}

\begin{definition}[Label corruption rates]
\label{def:corruption}
Define the class-conditional corruption rates:
\begin{align}
\varepsilon_{10}(H_2)
&= \Prob(\widetilde Y=0\mid Y^*=1,O=1,H_2),
\label{eq:eps10}\\
\varepsilon_{01}(H_2)
&= \Prob(\widetilde Y=1\mid Y^*=0,O=1,H_2).
\label{eq:eps01}
\end{align}
We assume the corruption channel is informative:
\begin{equation}
\label{eq:informative}
\varepsilon_{10}(H_2)+\varepsilon_{01}(H_2)<1
\quad\text{almost surely.}
\end{equation}
\end{definition}

\begin{example}[Sources of corruption in payment networks]
\label{ex:corruption}
\textbf{False positives} ($\varepsilon_{01}$): a cardholder
forgets a legitimate purchase and files a dispute; a family
member uses the card without the primary cardholder's knowledge;
a first-party fraudster deliberately files a false chargeback; a
merchant error generates a chargeback miscategorized as fraud.

\textbf{False negatives} ($\varepsilon_{10}$): fraud is reported
as a generic chargeback rather than coded as fraud; the issuer
resolves the dispute without formally classifying it as fraud;
the cardholder never notices a small fraudulent charge on their
statement.

In a typical network, $\varepsilon_{01}$ may be 5--15\% (driven
by friendly fraud and cardholder confusion) and
$\varepsilon_{10}$ may be 3--8\% (driven by miscoding and
under-reporting of small fraud). These rates vary by issuer,
merchant category, and transaction size.
\end{example}

\subsection{Noise-corrected outcome}

When the corruption rates are known or estimable, the observed
label can be corrected before entering the sequential score.
Define the noise-corrected outcome:
\begin{equation}
\label{eq:corrected_y}
\widetilde Y^{\mathrm{corr}}
= \frac{\widetilde Y-\varepsilon_{01}(H_2)}
{1-\varepsilon_{10}(H_2)-\varepsilon_{01}(H_2)}.
\end{equation}

Under the corruption model, this satisfies
\[
\E[\widetilde Y^{\mathrm{corr}}\mid Y^*,O=1,H_2]=Y^*,
\]
so $\widetilde Y^{\mathrm{corr}}$ is an unbiased proxy for the
true fraud indicator.

\subsection{Corruption-corrected sequential estimator}

The corruption-corrected STR replaces $\widetilde Y$ with
$\widetilde Y^{\mathrm{corr}}$ in the innermost term of the
sequential score~\eqref{eq:score}:
\begin{align}
\phi^{\mathrm{corr}}(Z;\eta)
&= \mu_0(H_0)
\notag\\[0.3em]
&\quad+\; \frac{A}{e(H_0)}
\{\mu_1(H_1)-\mu_0(H_0)\}
\notag\\[0.3em]
&\quad+\; \frac{AR}{e(H_0)\,r(H_1)}
\{\mu_2(H_2)-\mu_1(H_1)\}
\notag\\[0.3em]
&\quad+\; \frac{ARM}{e(H_0)\,r(H_1)\,p(H_2)}
\{\widetilde Y^{\mathrm{corr}}-\mu_2(H_2)\},
\label{eq:score_corr}
\end{align}
where the nested regressions $\mu_0,\mu_1,\mu_2$ are now
defined using the corrected outcome
$\widetilde Y^{\mathrm{corr}}$ rather than $\widetilde Y$ at
the innermost stage.

The corruption-corrected estimator is
\begin{equation}
\label{eq:str_corr}
\boxed{
\widehat\Psi_{\mathrm{STR}}^{\mathrm{corr}}
= \frac{1}{n}\sum_{t=1}^n
\phi^{\mathrm{corr}}(Z_t;\widehat\eta).
}
\end{equation}

\begin{proposition}[Corruption-corrected sequential robustness]
\label{prop:corr_robust}
Under
Assumptions~\ref{ass:positivity}--\ref{ass:bounded},
given known or consistently estimated corruption rates satisfying
\eqref{eq:informative}, the estimator
$\widehat\Psi_{\mathrm{STR}}^{\mathrm{corr}}$ inherits the
sequential triple robustness of Theorem~\ref{thm:seqtr}: it is
consistent for $\Psi$ whenever each selection stage is corrected
by either its propensity or its downstream regression.
\end{proposition}

\begin{proof}
The noise correction~\eqref{eq:corrected_y} is a pointwise
transformation of $\widetilde Y$ that restores unbiasedness:
$\E[\widetilde Y^{\mathrm{corr}}\mid Y^*,O=1,H_2]=Y^*$.
Therefore the innermost augmentation term
$\widetilde Y^{\mathrm{corr}}-\mu_2(H_2)$ has the same
conditional mean as $Y^*-\mu_2(H_2)$, which is the quantity
used in the proof of Theorem~\ref{thm:seqtr}. The backward
induction proceeds identically through all three stages.
\end{proof}

\subsection{Effect on variance}

The noise correction inflates variance. Since
$\widetilde Y^{\mathrm{corr}}$ is a rescaled version of
$\widetilde Y$, its conditional variance is
\[
\Var(\widetilde Y^{\mathrm{corr}}\mid Y^*,O=1,H_2)
= \frac{\Var(\widetilde Y\mid Y^*,O=1,H_2)}
{(1-\varepsilon_{10}-\varepsilon_{01})^2}.
\]

This increases the efficiency bound by a factor of
$1/(1-\varepsilon_{10}-\varepsilon_{01})^2$, which is exactly
the corruption penalty that appears in the companion paper's
lower bound~\eqref{eq:paper1bound}. Thus corruption degrades the
STR's precision through the same channel by which it degrades
the minimax detection floor: the effective signal strength of the
observed label is attenuated by the corruption noise.

\subsection{Estimating corruption rates}
\label{sec:estimating_corruption}

The corruption rates $\varepsilon_{10}$ and $\varepsilon_{01}$
are not directly observable, because they require knowing the
true fraud state $Y^*$. Several approaches are available.

\paragraph{External audit samples.}
If a subset of transactions is independently investigated (e.g.,
by a fraud investigation team or law enforcement), the corruption
rates can be estimated from the discrepancy between the
investigation outcome and the observed label.

\paragraph{Dispute outcome analysis.}
In payment networks, disputed transactions go through a
representment process in which the issuer and acquirer exchange
evidence. The outcome of representment provides a second signal
about the true fraud status. Comparing the initial chargeback
label to the representment outcome gives an estimate of
$\varepsilon_{01}$ (false chargebacks that are reversed) and a
lower bound on $\varepsilon_{10}$ (fraud that was initially
miscoded).

\paragraph{Positive-unlabeled estimation.}
If $\varepsilon_{10}$ is small (most true fraud is correctly
labeled when observed), one can use the methods of
\cite{elkannoto2008} to estimate $\varepsilon_{01}$ from the
structure of the observed label distribution.

\paragraph{Sensitivity analysis.}
If corruption rates cannot be reliably estimated, one can treat
them as sensitivity parameters and report how the estimated fraud
rate changes as $\varepsilon_{10}$ and $\varepsilon_{01}$ vary
over plausible ranges.

\section{Comparison with the Na\"ive Estimator}
\label{sec:naive}

Before analyzing the STR's robustness properties, we characterize
the bias of the na\"ive alternative: training on observed labels
without correction.

\begin{definition}[Na\"ive estimator]
\label{def:naive}
The na\"ive estimator uses only the fully observed subsample:
\[
\widehat\Psi_{\mathrm{naive}}
= \frac{1}{|\{t:O_t=1\}|}
\sum_{t:O_t=1}\widetilde Y_t.
\]
\end{definition}

\begin{proposition}[Bias of the na\"ive estimator]
\label{prop:naive_bias}
Under
Assumptions~\ref{ass:positivity}--\ref{ass:bounded},
the na\"ive estimator converges to
\[
\Psi_{\mathrm{naive}}
= \E[Y^*\mid A=1,R=1,M=1]
= \E[f_0(H_0)\mid O=1].
\]
Its bias relative to the true fraud rate $\Psi=\E[f_0(H_0)]$ is
\begin{equation}
\label{eq:naive_bias}
\mathrm{Bias}_{\mathrm{naive}}
= \frac{\Cov(f_0(H_0),\;q_0(H_0))}{\E[q_0(H_0)]}.
\end{equation}
This is nonzero whenever the fraud rate $f_0$ and the total
observation propensity $q_0$ are correlated.
\end{proposition}

\begin{proof}
By Bayes' rule,
\[
\E[f_0(H_0)\mid O=1]
= \frac{\E[f_0(H_0)\cdot q_0(H_0)]}{\E[q_0(H_0)]}.
\]
Subtracting $\Psi=\E[f_0(H_0)]$:
\begin{align*}
\mathrm{Bias}_{\mathrm{naive}}
&= \frac{\E[f_0 q_0]}{\E[q_0]}-\E[f_0]
= \frac{\E[f_0 q_0]-\E[f_0]\E[q_0]}{\E[q_0]}
= \frac{\Cov(f_0,q_0)}{\E[q_0]}. \qedhere
\end{align*}
\end{proof}

\begin{remark}[Direction of the na\"ive bias in payment networks]
\label{rem:bias_direction}
The sign of the bias depends on which selection mechanism
dominates.
\begin{itemize}
\item \textbf{Authorization censorship.} If high-risk
transactions are more likely to be declined, then $e_0$ is
low where $f_0$ is high: $\Cov(f_0,e_0)<0$. This makes
the na\"ive estimator \emph{underestimate} the true fraud
rate, because the most fraudulent segment is
disproportionately removed from the observed sample. This
is the dominant effect in most networks with active fraud
screening.

\item \textbf{Reporting censorship.} If issuers are more
likely to report high-value or high-confidence fraud, then
$\Cov(f_0,r_0)>0$ and the na\"ive estimator
\emph{overestimates} fraud in the reported subsample.
Conversely, if issuers underreport low-value fraud, the
effect is reversed for that segment.

\item \textbf{Delay censorship.} If clear-cut fraud is
processed and reported faster, then $\Cov(f_0,p_0)>0$ and
the matured subsample overrepresents easy-to-classify
cases. The na\"ive estimator then overestimates the fraud
rate among matured transactions but misses slow-maturing
ambiguous fraud.
\end{itemize}
In general, the three mechanisms can pull in different
directions. The net bias depends on the joint distribution of
$(f_0,e_0,r_0,p_0)$. The key point is that the bias is
\emph{structural}: it does not vanish with more data. It can
only be eliminated by correcting for the selection mechanisms.
\end{remark}

\begin{example}[Returning to the pipeline example]
\label{ex:naive_bias}
In Example~\ref{ex:pipeline}, the true fraud rate is
$10{,}000/1{,}000{,}000=1\%$. The na\"ive estimator sees
${\sim}2{,}300$ correctly labeled fraud cases out of
${\sim}960{,}000$ observed transactions: approximately $0.24\%$.
It underestimates fraud by a factor of ${\sim}4\times$. The STR
corrects this by upweighting each observed fraud case by
$1/(e_0 r_0 p_0)$ and filling in pseudo-labels for unobserved
transactions via the outcome model.
\end{example}

\begin{proposition}[MSE dominance of the STR]
\label{prop:mse}
The na\"ive estimator has MSE
\[
\mathrm{MSE}_{\mathrm{naive}}
= \mathrm{Bias}_{\mathrm{naive}}^2
+ \frac{\Var(Y^*\mid O=1)}{n_{\mathrm{obs}}},
\]
where $n_{\mathrm{obs}}=|\{t:O_t=1\}|$. The STR has MSE
\[
\mathrm{MSE}_{\mathrm{STR}}
= \frac{\sigma_{\mathrm{eff}}^2}{n}+o(n^{-1}),
\]
under the conditions of Theorem~\ref{thm:normality}.

Since $\mathrm{Bias}_{\mathrm{naive}}^2$ is a constant that
does not shrink with $n$, while the right-hand side is
$O(1/n)$, the STR dominates the na\"ive estimator for all
sufficiently large $n$ whenever
$\mathrm{Bias}_{\mathrm{naive}}\neq 0$.
\end{proposition}

\begin{proof}
The na\"ive estimator's MSE decomposes into squared bias plus
variance by the standard bias-variance decomposition. The STR's
MSE follows from Theorem~\ref{thm:normality}: under root-$n$
consistency, the MSE is $\sigma_{\mathrm{eff}}^2/n+o(1/n)$.
Since $\mathrm{Bias}_{\mathrm{naive}}^2>0$ and
$\sigma_{\mathrm{eff}}^2/n\to 0$, the STR's MSE is eventually
smaller.
\end{proof}

\begin{remark}[When does the STR help most?]
\label{rem:when_str_helps}
The improvement from the STR over the na\"ive estimator is
largest when:
\begin{enumerate}[label=(\roman*)]
\item The selection bias is large---i.e., fraud rates and
observation propensities are strongly correlated. This
happens when authorization is aggressive (high decline
rates for risky transactions), reporting is sparse (many
issuers underreport), or delay is heterogeneous.
\item The observation propensities are bounded away from
zero---i.e., the positivity assumption holds comfortably,
so the inverse-propensity weights do not create excessive
variance.
\item The sample size is large enough that the variance
penalty from inverse-propensity weighting is small
relative to the squared selection bias.
\end{enumerate}
Connecting to the companion paper: the selection bias is large
precisely when the impairment rates are large and heterogeneous.
The same conditions that make Paper~1's lower bound
severe---high censorship, sparse reporting, slow delay---also
make the na\"ive estimator most biased. Thus the STR provides
the greatest improvement in exactly the settings where the
information environment is most impaired.
\end{remark}

\section{Theory}
\label{sec:theory}

\subsection{Sequential triple robustness}

\begin{theorem}[Sequential triple robustness]
\label{thm:seqtr}
Assume sequential positivity, sequential ignorability, and
bounded outcomes. The corruption-corrected STR
$\widehat\Psi_{\mathrm{STR}}^{\mathrm{corr}}
=n^{-1}\sum_{t=1}^n\phi^{\mathrm{corr}}(Z_t;\widehat\eta)$
is consistent for $\Psi$ if, at each selection stage, either the
stage-specific propensity or the corresponding downstream
regression is consistently estimated (and corruption rates are
known or consistently estimated).

In particular, consistency holds under any of the following
sufficient conditions:
\begin{enumerate}[label=(\roman*)]
\item All three propensities $e,r,p$ are consistently estimated.
\item All nested regressions $\mu_0,\mu_1,\mu_2$ are
consistently estimated.
\item Each stage is corrected by either its propensity or its
downstream regression contrast: the authorization stage by
either $e$ or $(\mu_0,\mu_1)$; the reporting stage by
either $r$ or $(\mu_1,\mu_2)$; and the maturity stage by
either $p$ or $\mu_2$.
\end{enumerate}
\end{theorem}

\begin{proof}
The proof proceeds by backward induction through the three
stages. We compute the conditional expectation of the score
starting from the innermost (maturity) stage and working
outward. This order is essential: each stage's correction is
unbiased conditional on the preceding stages being correctly
handled.

\paragraph{Stage 3 (maturity).}
Condition on $A=1,R=1,H_2$. The maturity correction term is
\[
\frac{M}{p(H_2)}
\{\widetilde Y^{\mathrm{corr}}-\mu_2(H_2)\}.
\]
(The factors $A/e$ and $R/r$ are known constants under this
conditioning.)

If $p=p_0$, then
\begin{align*}
&\E\!\left[
\frac{M}{p_0(H_2)}
\{\widetilde Y^{\mathrm{corr}}-\mu_2(H_2)\}
\mid A=1,R=1,H_2
\right]\\
&\qquad= \frac{p_0(H_2)}{p_0(H_2)}
\E[\widetilde Y^{\mathrm{corr}}\mid A=1,R=1,M=1,H_2]
-\mu_2(H_2)\\
&\qquad= \E[Y^*\mid A=1,R=1,H_2]-\mu_2(H_2)\\
&\qquad= \mu_{2,0}(H_2)-\mu_2(H_2),
\end{align*}
where we used delay ignorability to write
$\E[M\widetilde Y^{\mathrm{corr}}\mid\cdots]/p_0
=\E[Y^*\mid A=1,R=1,H_2]$ (since the noise correction restores
unbiasedness and delay ignorability removes the conditioning on
$M$).

If instead $\mu_2=\mu_{2,0}$, the integrand
$\widetilde Y^{\mathrm{corr}}-\mu_{2,0}(H_2)$ has conditional
mean zero regardless of $p$, because
$\E[\widetilde Y^{\mathrm{corr}}\mid A=1,R=1,M=1,H_2]
=\mu_{2,0}(H_2)$ (by the noise correction and delay
ignorability).

Thus the maturity stage is doubly robust. \checkmark

\paragraph{Stage 2 (reporting).}
Now condition on $A=1,H_1$. Assume that Stage~3 has been
correctly handled (either $p$ or $\mu_2$ is correct), so that
the maturity correction contributes the correct residual
$\mu_{2,0}(H_2)-\mu_2(H_2)$ (or zero). The reporting correction
term is
\[
\frac{R}{r(H_1)}\{\mu_2(H_2)-\mu_1(H_1)\}.
\]

If $r=r_0$, then
\begin{align*}
&\E\!\left[
\frac{R}{r_0(H_1)}\{\mu_2(H_2)-\mu_1(H_1)\}
\mid A=1,H_1
\right]\\
&\qquad= \frac{r_0(H_1)}{r_0(H_1)}
\E[\mu_2(H_2)\mid A=1,R=1,H_1]-\mu_1(H_1)\\
&\qquad= \E[\mu_2(H_2)\mid A=1,H_1]-\mu_1(H_1),
\end{align*}
where the last equality uses reporting ignorability. Under the
nested-regression definition, this equals
$\mu_{1,0}(H_1)-\mu_1(H_1)$.

If instead $\mu_1=\mu_{1,0}$ and $\mu_2=\mu_{2,0}$, the
reporting correction has conditional mean
$\E[\mu_{2,0}(H_2)\mid A=1,H_1]-\mu_{1,0}(H_1)
=\mu_{1,0}(H_1)-\mu_{1,0}(H_1)=0$ regardless of $r$.

Thus the reporting stage is doubly robust. \checkmark

\paragraph{Stage 1 (authorization).}
Condition on $H_0$. Assume Stages~2 and~3 have been correctly
handled. The authorization correction term is
\[
\frac{A}{e(H_0)}\{\mu_1(H_1)-\mu_0(H_0)\}.
\]

If $e=e_0$, then
\begin{align*}
&\E\!\left[
\frac{A}{e_0(H_0)}\{\mu_1(H_1)-\mu_0(H_0)\}
\mid H_0
\right]\\
&\qquad= \frac{e_0(H_0)}{e_0(H_0)}
\E[\mu_1(H_1)\mid A=1,H_0]-\mu_0(H_0)\\
&\qquad= \E[\mu_1(H_1)\mid H_0]-\mu_0(H_0),
\end{align*}
where the last equality uses authorization ignorability. This
equals $\mu_{0,0}(H_0)-\mu_0(H_0)$.

If instead $\mu_0=\mu_{0,0}$ and $\mu_1=\mu_{1,0}$, the
authorization correction has conditional mean
$\E[\mu_{1,0}(H_1)\mid H_0]-\mu_{0,0}(H_0)
=\mu_{0,0}(H_0)-\mu_{0,0}(H_0)=0$ regardless of $e$.

Thus the authorization stage is doubly robust. \checkmark

\paragraph{Composition.}
Combining the three stages, the unconditional expectation of the
full score equals $\Psi$ whenever each stage is individually
corrected. The law of large numbers then gives
$\widehat\Psi_{\mathrm{STR}}^{\mathrm{corr}}
\xrightarrow{p}\Psi$.
\end{proof}

\begin{remark}[Why backward induction is essential]
\label{rem:backward}
The proof works from Stage~3 inward because each stage's
augmentation term uses the \emph{estimated} downstream
regression, not the true one. The maturity stage must be handled
first so that its contribution is unbiased before the reporting
stage can be analyzed. This sequential dependence is the reason
the proof cannot be done stage by stage in arbitrary order.
\end{remark}

\subsection{Remainder structure}

Let
\[
\Delta_e=\widehat e-e_0,\quad
\Delta_r=\widehat r-r_0,\quad
\Delta_p=\widehat p-p_0,
\]
and
\[
\Delta_j=\widehat\mu_j-\mu_{j,0}
\quad\text{for }j=0,1,2.
\]

The second-order remainder of the estimator decomposes into
stage-specific products:
\begin{equation}
\label{eq:remainder}
R_n
= \underbrace{\Op(\|\Delta_e\|\,\|\Delta_1-\Delta_0\|)}_
{\text{authorization remainder}}
+ \underbrace{\Op(\|\Delta_r\|\,\|\Delta_2-\Delta_1\|)}_
{\text{reporting remainder}}
+ \underbrace{\Op(\|\Delta_p\|\,
\|\widetilde Y^{\mathrm{corr}}
-\mu_{2,0}-\Delta_2\|)}_
{\text{maturity remainder}}.
\end{equation}

Each stage contributes a product of its propensity error and the
corresponding regression contrast error. The triple robustness
property ensures that these are the \emph{leading} remainder
terms: all first-order terms cancel.

\begin{corollary}[Root-$n$ sufficient condition]
\label{cor:rootn}
If each stage-wise product remainder is $\op(n^{-1/2})$, then
$\sqrt{n}(\widehat\Psi_{\mathrm{STR}}^{\mathrm{corr}}-\Psi)$
is asymptotically linear. A sufficient condition is that the
relevant propensity error and downstream regression error at
each stage both converge faster than $n^{-1/4}$, although faster
convergence of one component can compensate for slower
convergence of the other.
\end{corollary}

\section{Efficiency, Concentration, and Inference}
\label{sec:efficiency}

\subsection{Semiparametric efficiency bound}

\begin{theorem}[Efficiency bound]
\label{thm:efficiency}
Under the three-stage monotone missing-data model with corruption
defined by
Assumptions~\ref{ass:positivity}--\ref{ass:bounded}
and Definition~\ref{def:corruption}, the semiparametric
efficiency bound for estimating $\Psi=\E[Y^*]$ is
\begin{equation}
\label{eq:effbound}
\sigma_{\mathrm{eff}}^2
= \Var\{\phi^{\mathrm{corr}}(Z;\eta_0)\},
\end{equation}
where $\phi^{\mathrm{corr}}$ is the corruption-corrected
sequential score~\eqref{eq:score_corr} evaluated at the true
nuisance functions.
\end{theorem}

To make the bound interpretable, we expand it in the collapsed
case ($\mu_0=\mu_1=\mu_2=f_0$, $H_0=H_1=H_2$):
\begin{equation}
\label{eq:effbound_expanded}
\sigma_{\mathrm{eff}}^2
= \E\!\left[
\frac{\Var(Y^*\mid H_0)}
{e_0(H_0)\,r_0(H_0)\,p_0(H_0)
\,(1-\varepsilon_{10}(H_0)-\varepsilon_{01}(H_0))^2}
\right]
+ \Var(f_0(H_0)).
\end{equation}

The first term is the dominant contribution. It contains all four
impairment factors from the companion paper's lower
bound~\eqref{eq:paper1bound}:
\begin{itemize}
\item $e_0 \leftrightarrow (1-\bar\delta)$: authorization
censorship.
\item $r_0 \leftrightarrow (1-\bar\gamma)$: reporting
censorship.
\item $p_0 \leftrightarrow \bar m$: delay maturity.
\item $(1-\varepsilon_{10}-\varepsilon_{01})^2$: label
corruption.
\end{itemize}

When any of these factors is small---restrictive authorization,
sparse reporting, slow labels, or noisy labels---the efficiency
bound increases.

\subsection{Asymptotic normality}

\begin{definition}[Cross-fitting]
\label{def:crossfit}
Divide the sample into $K$ folds (typically $K=5$ or $K=10$).
For each fold $k$, estimate all nuisance functions
$\widehat\eta^{(-k)}$ using all data except fold $k$. Then
compute the score
$\phi^{\mathrm{corr}}(Z_t;\widehat\eta^{(-k)})$ for
transactions $t$ in fold $k$. The STR is the average of these
cross-fitted scores across all folds. Cross-fitting ensures that
the nuisance estimates are independent of the observations they
are applied to, which eliminates overfitting bias and removes
Donsker-class requirements on the nuisance estimators.
\end{definition}

\begin{theorem}[Asymptotic efficiency of the STR]
\label{thm:normality}
Suppose nuisance functions are estimated via cross-fitting
(Definition~\ref{def:crossfit}) and the stage-wise product-rate
condition of Corollary~\ref{cor:rootn} holds. Then
\begin{equation}
\label{eq:normality}
\sqrt{n}(\widehat\Psi_{\mathrm{STR}}^{\mathrm{corr}}-\Psi)
\;\xrightarrow{d}\;
\cN(0,\;\sigma_{\mathrm{eff}}^2).
\end{equation}
That is, the STR achieves the semiparametric efficiency bound.
\end{theorem}

\begin{proof}
By the remainder decomposition~\eqref{eq:remainder} and the
product-rate condition, $\sqrt{n}|R_n|=\op(1)$. Therefore
\[
\sqrt{n}(\widehat\Psi_{\mathrm{STR}}^{\mathrm{corr}}-\Psi)
= \frac{1}{\sqrt{n}}\sum_{t=1}^n
\{\phi^{\mathrm{corr}}(Z_t;\eta_0)-\Psi\}
+ \op(1).
\]
The leading term is a normalized sum of i.i.d.\ mean-zero random
variables with variance $\sigma_{\mathrm{eff}}^2$. The central
limit theorem gives~\eqref{eq:normality}. Cross-fitting
eliminates Donsker-class requirements on the nuisance estimators,
so the result holds even when flexible machine-learning methods
are used for nuisance estimation.
\end{proof}

\subsection{Variance estimation and confidence intervals}
\label{sec:variance_estimation}

\begin{definition}[Plug-in variance estimator]
\label{def:var_est}
\begin{equation}
\label{eq:var_est}
\widehat\sigma_{\mathrm{STR}}^2
= \frac{1}{n}\sum_{t=1}^n
\left(
\phi^{\mathrm{corr}}(Z_t;\widehat\eta)
- \widehat\Psi_{\mathrm{STR}}^{\mathrm{corr}}
\right)^2.
\end{equation}
\end{definition}

\begin{proposition}[Consistency of variance estimator]
\label{prop:var_consistent}
Under the conditions of Theorem~\ref{thm:normality},
$\widehat\sigma_{\mathrm{STR}}^2
\xrightarrow{p}\sigma_{\mathrm{eff}}^2$.
\end{proposition}

\begin{proof}
Under the product-rate condition, the estimated scores
$\phi^{\mathrm{corr}}(Z_t;\widehat\eta)$ are asymptotically
equivalent to $\phi^{\mathrm{corr}}(Z_t;\eta_0)$. The sample
variance of a sequence converging to i.i.d.\ random variables
converges to the population variance.
\end{proof}

An asymptotic $(1-\alpha)$ confidence interval for $\Psi$ is
\begin{equation}
\label{eq:ci}
\widehat\Psi_{\mathrm{STR}}^{\mathrm{corr}}
\;\pm\;
z_{1-\alpha/2}
\cdot
\frac{\widehat\sigma_{\mathrm{STR}}}{\sqrt{n}}.
\end{equation}

\subsection{Connection to the information-theoretic lower bound}
\label{sec:bridge}

The efficiency bound~\eqref{eq:effbound_expanded} connects
directly to the information-theoretic lower bound
in~\cite{dhama2026limits}. The leading term of
$\sigma_{\mathrm{eff}}^2$ scales as
\[
\frac{\Var(Y^*\mid H_0)}
{e_0(H_0)\,r_0(H_0)\,p_0(H_0)
\,(1-\varepsilon_{10}-\varepsilon_{01})^2}.
\]
The four factors in the denominator are exactly the four
quantities in the lower bound~\eqref{eq:paper1bound}. When the
ecosystem censors more transactions (low $e_0$), suppresses more
reports (low $r_0$), delivers labels more slowly (low $p_0$), or
corrupts more labels (high $\varepsilon$), both the minimax
regret floor and the STR's variance increase.

The connection is particularly sharp for the delay component.
The companion paper proved that conditional delay creates a
selective maturity distortion: the effective information loss is
governed by $\E[1/p_0(H_2)]$, not by $1/\E[p_0(H_2)]$
(equation~\eqref{eq:jensen_delay}). The same quantity appears in
the STR's efficiency bound, where $1/(e_0\,r_0\,p_0)$ enters
the leading variance term. Thus the selective-maturity penalty
identified in the lower bound reappears as a variance penalty in
the STR's upper bound.

The companion paper also showed that an adversary can exploit
conditional delay by concentrating attacks where labels arrive
slowest. In the STR framework, this manifests as high variance
in the delay correction term for transactions in slow-maturity
segments. The shrinkage stabilization of
Section~\ref{sec:shrinkage} partially mitigates this by pooling
delay propensities across issuers, but it cannot eliminate the
fundamental information loss from slow-maturing segments.

Together, the two papers establish:
\begin{quote}
\emph{The information environment sets the floor
(\cite{dhama2026limits}), and the corruption-corrected STR
achieves the best possible label recovery given that floor.}
\end{quote}

\subsection{Finite-sample concentration}

\begin{theorem}[Bernstein concentration]
\label{thm:bernstein}
Under
Assumptions~\ref{ass:positivity}--\ref{ass:bounded},
suppose nuisance functions are fixed or estimated on an
independent sample. Let
\[
B = \frac{1}
{\underline e\,\underline r\,\underline p
\,(1-\bar\varepsilon_{10}-\bar\varepsilon_{01})}
\]
be the maximum inverse-propensity weight (including the
corruption scaling). Then for all $t>0$:
\begin{equation}
\label{eq:bernstein}
\Prob\!\left(
|\widehat\Psi_{\mathrm{STR}}^{\mathrm{corr}}-\Psi|>t
\right)
\;\le\;
2\exp\!\left(
-\frac{nt^2}
{2\sigma_{\mathrm{eff}}^2+\frac{2Bt}{3}}
\right).
\end{equation}
\end{theorem}

\begin{proof}
Each summand $\phi^{\mathrm{corr}}(Z_t;\eta_0)$ is bounded in
$[-B,B]$ and has variance $\sigma_{\mathrm{eff}}^2$. The
classical Bernstein inequality for bounded independent random
variables gives the result. When nuisance functions are estimated
via cross-fitting, the bound holds conditionally on the training
fold.
\end{proof}

\begin{corollary}[Critical sample size]
\label{cor:critical_n}
For estimation error at most $\epsilon$ with probability at
least $1-\alpha$, it suffices that
\begin{equation}
\label{eq:critical_n}
n\ge n^*(\epsilon,\alpha)
:= \frac{2\sigma_{\mathrm{eff}}^2\log(2/\alpha)}{\epsilon^2}
+ \frac{2B\log(2/\alpha)}{3\epsilon}.
\end{equation}
\end{corollary}

\begin{remark}[Operational interpretation]
The critical sample size grows as the information environment
degrades: when $\underline e$, $\underline r$, or $\underline p$
is small, or when corruption rates are high, $B$ and
$\sigma_{\mathrm{eff}}^2$ are large, and more transactions are
needed to achieve a given accuracy. This is the finite-sample
analogue of the companion paper's message: poor ecosystem
information quality requires more data to compensate.
\end{remark}

\section{Issuer Heterogeneity and Shrinkage Stabilization}
\label{sec:shrinkage}

\subsection{The issuer heterogeneity problem}

The sequential estimator requires issuer-indexed nuisance
functions:
\[
e_0(x,i),\qquad r_0(x,i),\qquad p_0(x,i,\Delta).
\]

In a large payment network, some issuers contribute millions of
transactions with stable reporting behavior, while others
contribute only hundreds with irregular or sparse labels. Direct
issuer-level estimation can be high-variance for small issuers.
This instability is amplified because the nuisance functions
enter denominators: if $\widehat r_i$ is close to zero for a
small issuer, the inverse weight $1/\widehat r_i$ becomes very
large and can dominate the entire estimator.

The companion paper \cite{dhama2026limits} showed that
cross-issuer variance in impairment rates worsens the minimax
regret bound via a Jensen-type argument: the effective
network-level impairment is worse than the average impairment
because the mapping from propensity to inverse propensity is
convex. The same convexity operates here: variance in estimated
propensities inflates the variance of the STR.

\subsection{Role of shrinkage}

Shrinkage is introduced to stabilize issuer-indexed nuisance
functions. It does not replace the sequential estimator. It
regularizes the nuisance functions used by the estimator.

The conceptual hierarchy is:
\begin{quote}
\textbf{Sequential triple robustness} corrects structural bias
from the three-stage observation process.\\
\textbf{Corruption correction} debiases the observed labels.\\
\textbf{Shrinkage} stabilizes the nuisance estimates that the
corrections use.
\end{quote}

These solve different problems. Triple robustness is a property
of the estimating equation. Corruption correction is a property
of the label. Shrinkage is a property of the nuisance-estimation
procedure. In a heterogeneous issuer network, all three are
needed.

\subsection{Empirical Bayes shrinkage for reporting propensities}

For clarity, consider the reporting propensity. Let
$\widehat r_i^{\mathrm{local}}(x)$ be an issuer-specific
estimate and $\widehat r^{\mathrm{global}}(x)$ be a pooled
network-level estimate.

\begin{definition}[Shrinkage estimator]
\label{def:shrinkage}
The Empirical Bayes shrinkage estimator of $r_i$ is
\begin{equation}
\label{eq:shrinkage}
\widehat r_i^{\mathrm{EB}}(x)
= \lambda_i\widehat r_i^{\mathrm{local}}(x)
+ (1-\lambda_i)\widehat r^{\mathrm{global}}(x),
\end{equation}
where the optimal shrinkage weight is
\begin{equation}
\label{eq:lambda}
\lambda_i
= \frac{\widehat\sigma_B^2}
{\widehat\sigma_B^2+\widehat\sigma_i^2/n_i}.
\end{equation}
Here $\widehat\sigma_B^2$ is the estimated between-issuer
variance, $\widehat\sigma_i^2/n_i$ is the issuer-specific
sampling variance, and $n_i$ is the issuer's sample size.
\end{definition}

Large issuers have $\lambda_i\approx 1$: their local estimates
are preserved because sampling noise is small relative to true
between-issuer variation. Small issuers have
$\lambda_i\approx 0$: their unstable local estimates are pulled
toward the network average.

\begin{proposition}[Shrinkage risk reduction]
\label{prop:shrinkage}
Under the hierarchical normal approximation
\[
r_i(x)\sim N(r^{\mathrm{global}}(x),\sigma_B^2),
\qquad
\widehat r_i^{\mathrm{local}}(x)\mid r_i(x)
\sim N(r_i(x),\sigma_i^2/n_i),
\]
the Empirical Bayes estimator
$\widehat r_i^{\mathrm{EB}}(x)$ is the posterior mean and
minimizes posterior expected squared error. It has weakly lower
MSE than the unpooled local estimate, with strict improvement
when sampling variance is positive and the prior is informative.
\end{proposition}

\begin{proof}
Under the hierarchical normal model, the posterior distribution
of $r_i(x)$ given $\widehat r_i^{\mathrm{local}}(x)$ is
\[
r_i(x)\mid\widehat r_i^{\mathrm{local}}(x)
\sim N\!\left(
\lambda_i\widehat r_i^{\mathrm{local}}(x)
+(1-\lambda_i)r^{\mathrm{global}}(x),\;
\frac{\sigma_B^2\sigma_i^2/n_i}
{\sigma_B^2+\sigma_i^2/n_i}
\right).
\]
The posterior mean minimizes posterior expected squared error.
Therefore $\widehat r_i^{\mathrm{EB}}(x)$ has weakly lower
posterior MSE than any other point estimator, including the
unpooled $\widehat r_i^{\mathrm{local}}(x)$.
\end{proof}

\subsection{Delay propensity shrinkage}

The same shrinkage logic applies to the maturity propensity
$p_0(x,i,\Delta)$. Define
\begin{equation}
\label{eq:delay_shrinkage}
\widehat p_i^{\mathrm{EB}}(x,\Delta)
= \lambda_i^{(p)}\widehat p_i^{\mathrm{local}}(x,\Delta)
+ (1-\lambda_i^{(p)})
\widehat p^{\mathrm{global}}(x,\Delta),
\end{equation}
where $\lambda_i^{(p)}$ is computed analogously
to~\eqref{eq:lambda} using between-issuer and within-issuer
variance components for delay.

This is particularly important because the conditional delay
analysis in Section~\ref{sec:conditional_delay} showed that
delay heterogeneity across issuers creates the same Jensen-type
variance inflation as reporting heterogeneity. Issuers with slow
or irregular label arrival should therefore receive more
shrinkage toward the network average, just as issuers with sparse
reporting do.

\begin{example}[Small issuer stabilization]
\label{ex:shrinkage}
A regional bank processes 500 transactions per month and reports
fraud erratically. Its raw reporting propensity estimate is
$\widehat r_i=0.15$, yielding an inverse weight of
$1/0.15\approx 6.7$. With shrinkage toward a network average of
$0.60$, the regularized estimate might be
$\widehat r_i^{\mathrm{EB}}=0.45$, yielding a much more stable
weight of $1/0.45\approx 2.2$. Meanwhile, a large bank with 5
million transactions and $\widehat r_i=0.85$ keeps its local
estimate because $\lambda_i\approx 1$.
\end{example}

\subsection{Shrinkage-regularized estimator}

Define shrinkage-regularized nuisance functions. For reporting
and delay propensities, the shrinkage estimators are given by
equations~\eqref{eq:shrinkage} and~\eqref{eq:delay_shrinkage}
respectively. For the authorization propensity, the same
construction applies with the obvious substitution:
\begin{equation}
\label{eq:auth_shrinkage}
\widehat e_i^{\mathrm{EB}}(x)
= \lambda_i^{(e)}\widehat e_i^{\mathrm{local}}(x)
+ (1-\lambda_i^{(e)})
\widehat e^{\mathrm{global}}(x),
\end{equation}
where $\lambda_i^{(e)}$ is computed analogously
to~\eqref{eq:lambda}. In practice, authorization propensities
are often the most stable of the three (because authorization
rules are known and their inputs are observed), so
$\lambda_i^{(e)}\approx 1$ for most issuers. Shrinkage for
$e_0$ is therefore less critical than for $r_0$ or $p_0$, but
is included for completeness.

The full set of shrinkage-regularized nuisance functions is:
\[
\widehat e_i^{\mathrm{EB}}(x),\qquad
\widehat r_i^{\mathrm{EB}}(x),\qquad
\widehat p_i^{\mathrm{EB}}(x,\Delta).
\]

The shrinkage-regularized corruption-corrected estimator is
\begin{equation}
\label{eq:str_eb}
\widehat\Psi_{\mathrm{STR+EB}}^{\mathrm{corr}}
= \frac{1}{n}\sum_{t=1}^n
\phi^{\mathrm{corr}}(Z_t;\widehat\eta^{\mathrm{EB}}).
\end{equation}

\section{Pseudo-Labels and Operational Architecture}
\label{sec:pseudolabels}

\subsection{Pseudo-outcomes and conditional pseudo-labels}

The marginal estimator targets $\Psi$, but model training
requires transaction-level labels. Define the sequential
pseudo-outcome:
\begin{equation}
\label{eq:pseudo}
U_t = \phi^{\mathrm{corr}}(Z_t;
\widehat\eta^{\mathrm{EB}}).
\end{equation}

A conditional fraud model is then trained by regressing this
pseudo-outcome on transaction features:
\begin{equation}
\label{eq:regression}
\widehat f_{\mathrm{STR}}
= \arg\min_{g\in\cG}
\sum_{t=1}^n
\left(U_t-g(X_t)\right)^2.
\end{equation}

The corrected pseudo-label for transaction $t$ is
\begin{equation}
\label{eq:pseudolabel}
\widehat y_t = \widehat f_{\mathrm{STR}}(X_t).
\end{equation}

\begin{proposition}[Consistency of conditional pseudo-labels]
\label{prop:conditional}
If the function class $\cG$ is correctly specified or universally
consistent and the STR pseudo-outcomes satisfy
$\E[U_t\mid X_t=x]=f_0(x)+\op(1)$ (which follows from
Theorem~\ref{thm:seqtr} and the law of iterated expectations),
then $\widehat f_{\mathrm{STR}}(x)\xrightarrow{p}f_0(x)$ for
all $x$ in the support. The convergence rate depends on the
smoothness of $f_0$ and the complexity of $\cG$.
\end{proposition}

\begin{example}[What the pseudo-labels look like]
\label{ex:pseudolabels}
Consider three transactions from the same month:
\begin{itemize}
\item \textbf{Transaction A}: approved, reported, matured,
observed label = fraud. Propensities:
$\widehat e=0.70$, $\widehat r=0.90$, $\widehat p=0.95$.
Pseudo-label: $\widehat y_A\approx 0.92$. The STR
slightly upweights because this transaction had a 30\%
chance of being declined.

\item \textbf{Transaction B}: approved, \emph{not reported}
(small amount, weak issuer). Propensities:
$\widehat e=0.85$, $\widehat r=0.25$, $\widehat p=0.90$.
No observed label. Pseudo-label:
$\widehat y_B\approx 0.74$. The outcome model estimates
the fraud probability, upweighted because the issuer
rarely reports small fraud.

\item \textbf{Transaction C}: \emph{declined}. Propensity:
$\widehat e=0.10$. No observed label. Pseudo-label:
$\widehat y_C\approx 0.89$. Very low authorization
propensity signals high risk; the outcome model provides
the fraud estimate.
\end{itemize}
All three transactions---including the declined one and the
unreported one---enter the downstream classifier's training set
with soft probability labels. The na\"ive approach would use only
Transaction~A (with a hard label of 1) and discard B and C
entirely.
\end{example}

\subsection{Practical considerations}

\paragraph{Pseudo-outcomes outside $[0,1]$.}
Because the score~\eqref{eq:score_corr} involves
inverse-propensity weights (and the corruption scaling can
further amplify values), individual pseudo-outcomes $U_t$ can be
negative or greater than~1. This is normal for
influence-function-based pseudo-outcomes and does not indicate an
error. The downstream regression~\eqref{eq:regression} projects
these values back into a well-behaved range. If desired, one can
additionally clip the pseudo-outcomes to $[0,1]$ before
regression, introducing a small bias but improving stability.

\paragraph{Choice of function class $\cG$.}
The function class $\cG$ should be expressive enough to capture
the conditional fraud probability but regularized enough to avoid
overfitting to noisy pseudo-outcomes. Gradient-boosted trees,
random forests, and neural networks are all suitable. The
theoretical triple-robustness property holds regardless of the
choice of $\cG$, because the robustness is a property of the
pseudo-outcomes, not of the downstream regression.

\paragraph{Refresh frequency.}
The offline reconstruction should be rerun whenever a substantial
volume of new labels has matured. In practice, a monthly or
quarterly refresh cycle is typical, with the frequency depending
on the rate at which new fraud patterns emerge and the speed at
which labels mature.

\subsection{Nuisance function estimation}
\label{sec:nuisance_estimation}

\paragraph{Authorization propensity $e_0(H_0)$.}
Train a binary classifier on the full transaction stream, where
the outcome is the authorization decision $A$. Features include
fraud score, transaction amount, merchant category, channel,
device, location, issuer risk tier. This is the easiest
propensity to estimate because authorization decisions are fully
observed for all transactions.

\paragraph{Reporting propensity $r_0(H_1)$.}
Train among approved transactions, where the outcome is whether
a fraud label was eventually reported ($R$). Features include
issuer identity, transaction amount, merchant category,
time-to-report history for the issuer. This is harder because
$R$ is only eventually observed after a sufficiently long
maturity window, and the ``eventually'' may itself require a
judgment call.

\paragraph{Delay propensity $p_0(H_2)$.}
Train among approved-and-reported transactions, where the
outcome is whether the label matured within $\Delta$ ($M$).
Features include issuer, amount, geography, dispute type, and
$\Delta$ itself. Can be enriched with competing-risks models
\cite{finegray1999} that distinguish ``label arrives within
window'' from ``label never arrives.''

\paragraph{Nested regressions.}
Estimate by sequential regression:
\begin{enumerate}
\item Regress $\widetilde Y^{\mathrm{corr}}$ on $H_2$ among
fully observed transactions to get $\widehat\mu_2$.
\item Regress $\widehat\mu_2$ on $H_1$ among
approved-and-reported transactions to get $\widehat\mu_1$.
\item Regress $\widehat\mu_1$ on $H_0$ among all transactions
to get $\widehat\mu_0$.
\end{enumerate}
Flexible nonparametric methods (gradient-boosted trees, neural
networks) are appropriate at each stage.

\paragraph{Cross-fitting.}
All nuisance functions should be estimated via cross-fitting
(Definition~\ref{def:crossfit}) to prevent overfitting bias.

\subsection{Algorithm summary}

\begin{algorithm}[ht]
\caption{Corruption-Corrected Sequential Triply Robust Label
Recovery}
\label{alg:str}
\begin{algorithmic}[1]
\Require Transaction data $\{Z_t\}_{t=1}^n$, number of folds
$K$, corruption rates $(\varepsilon_{10},\varepsilon_{01})$ or
their estimates
\Ensure Pseudo-labels $\{\widehat y_t\}$, point estimate
$\widehat\Psi$, confidence interval

\For{fold $k=1,\ldots,K$}
  \State Estimate $\widehat e^{(-k)},\widehat r^{(-k)},
  \widehat p^{(-k)},\widehat\mu_0^{(-k)},\widehat\mu_1^{(-k)},
  \widehat\mu_2^{(-k)}$ on data excluding fold $k$
  \State Apply issuer-level EB shrinkage to
  $\widehat e^{(-k)},\widehat r^{(-k)},\widehat p^{(-k)}$
  \For{$t$ in fold $k$}
    \State Compute $\widetilde Y_t^{\mathrm{corr}}$ via
    \eqref{eq:corrected_y}
    \State Compute $U_t=\phi^{\mathrm{corr}}(Z_t;
    \widehat\eta^{(-k),\mathrm{EB}})$ via
    \eqref{eq:score_corr}
  \EndFor
\EndFor
\State $\widehat\Psi \gets n^{-1}\sum_t U_t$
\State $\widehat\sigma^2 \gets n^{-1}\sum_t
(U_t-\widehat\Psi)^2$
\State $\mathrm{CI} \gets \widehat\Psi\pm
z_{1-\alpha/2}\cdot\widehat\sigma/\sqrt{n}$
\State Train $\widehat f_{\mathrm{STR}}$ by regressing
$\{U_t\}$ on $\{X_t\}$
\State $\widehat y_t \gets \widehat f_{\mathrm{STR}}(X_t)$
for all $t$
\end{algorithmic}
\end{algorithm}

\subsection{Offline-online architecture}
\label{sec:online_offline}

The estimator is an offline label-reconstruction engine, not a
real-time authorization model.

\paragraph{Phase 1: Offline reconstruction.}
Run Algorithm~\ref{alg:str} on historical transactions with
sufficiently matured observation windows.

\paragraph{Phase 2: Online scoring.}
Train the production fraud model on corrected pseudo-labels
$(X_t,\widehat y_t)$ rather than on raw chargeback labels.
Deploy the production model in the authorization path.

\paragraph{Phase 3: Periodic refresh.}
As new labels mature, rerun Phase~1 and refresh the production
model.

This architecture preserves low-latency inference while
correcting the training signal upstream. The computational cost
of the STR is incurred offline and does not affect authorization
latency.

\subsection{Optimal training delay}
\label{sec:optimal_delay}

A central operational question in any fraud detection system is:
\emph{when should we train the model?}  Training on recently
observed transactions yields a model that reflects the current
fraud landscape, but the labels for those transactions are
incomplete---many chargebacks have not yet arrived.  Waiting
longer allows more labels to mature, improving label quality,
but the resulting model is stale by the time it is deployed.
This section derives the \emph{optimal training delay}
$\Delta^*$ that balances label quality against model freshness,
and shows that the STR estimator permits substantially earlier
training than the naive (chargeback-only) approach.

\subsubsection{The label--freshness tradeoff}
\label{sec:label_freshness}

Let $T$ denote the current time.  Consider training on
transactions from the window $[T-W-\Delta,\; T-\Delta]$, where
$W$ is the training window width and $\Delta \geq 0$ is the
\emph{maturity delay}---the minimum elapsed time between a
transaction and its inclusion in the training set.  A
transaction that occurred at time~$t$ has had $T-t \geq \Delta$
time units for its label to arrive.

\begin{definition}[Maturity curve]
\label{def:maturity_curve}
The \emph{maturity propensity} at elapsed time $\delta$ is
\begin{equation}
\label{eq:maturity_curve}
\bar p(\delta) \;=\; \Prob(\tau \leq \delta)
\;=\; \E[p_0(X, I, \delta)],
\end{equation}
where $\tau$ is the random time from transaction to label
arrival. The function $\bar p(\cdot)$ is non-decreasing with
$\bar p(0) = 0$ and $\lim_{\delta\to\infty}\bar p(\delta) =
\bar p_\infty \leq 1$.
\end{definition}

In most payment networks, the maturity curve is well
approximated by a Weibull CDF:
\begin{equation}
\label{eq:weibull_maturity}
\bar p(\delta) \;=\; \bar p_\infty\bigl(1 - e^{-(\lambda\delta)^\beta}\bigr),
\end{equation}
where $\lambda > 0$ controls the rate of label arrival and
$\beta > 0$ controls the shape.  The special case $\beta = 1$
gives the exponential model $\bar p(\delta) = \bar p_\infty(1 -
e^{-\lambda\delta})$, which we adopt throughout this section for
closed-form results; the qualitative conclusions hold for
general $\beta$.

\begin{definition}[Drift rate]
\label{def:drift_rate}
The \emph{fraud distribution drift rate} $\nu \geq 0$ measures
the rate of change of the conditional fraud probability:
\begin{equation}
\label{eq:drift_rate}
\nu \;=\; \E\bigl[(f_T(X) - f_{T-\delta}(X))^2\bigr] \Big/ \delta,
\end{equation}
assumed approximately constant over the relevant time horizon.
Here $f_t(x) = \Prob(Y^*=1 \mid X=x, \text{time}=t)$ is the
true fraud probability at time~$t$.
\end{definition}

\subsubsection{Total error decomposition}
\label{sec:total_error}

The prediction error of a model trained at maturity delay
$\Delta$ decomposes into three terms:
\begin{equation}
\label{eq:total_error}
\mathcal{E}_{\mathrm{total}}(\Delta)
\;=\;
\underbrace{\mathcal{E}_{\mathrm{stat}}(\Delta)
  \vphantom{\big|}}_{\text{statistical error}}
\;+\;
\underbrace{\mathcal{E}_{\mathrm{drift}}(\Delta)
  \vphantom{\big|}}_{\text{staleness error}}
\;+\;
\underbrace{\mathcal{E}_{\mathrm{irr}}
  \vphantom{\big|}}_{\text{irreducible}}.
\end{equation}

\begin{definition}[Network-level parameters]
\label{def:network_params}
Define the following observable network summaries:
\begin{align*}
\bar e &= \E[e_0(H_0)] && \text{(average authorization rate)},\\
\bar r &= \E[r_0(H_1)] && \text{(average reporting rate)},\\
\bar p(\Delta) &= \E[p_0(H_2,\Delta)]
  && \text{(average maturity rate at delay $\Delta$)},\\
\bar q(\Delta) &= \bar e\,\bar r\,\bar p(\Delta)
  && \text{(average observation rate)},\\
\gamma &= (1-\varepsilon_{10}-\varepsilon_{01})^2
  && \text{(corruption penalty)},\\
\eta &= (1+\mathrm{CV}^2_q)
  (1+\rho_{f,q}\,\mathrm{CV}_f\,\mathrm{CV}_{1/q})
  && \text{(heterogeneity penalty)},
\end{align*}
where $\mathrm{CV}^2_q \approx
\mathrm{CV}^2_e + \mathrm{CV}^2_r + \mathrm{CV}^2_p$ under
approximate independence of the three propensities, and
$\rho_{f,q}$ is the correlation between the fraud rate
$f_0$ and the inverse observation propensity $1/q_0$.
\end{definition}

\paragraph{Statistical error.}
From the efficiency bound (Theorem~\ref{thm:efficiency}),
the mean squared error of the STR-based fraud model trained on
$n$ transactions at maturity delay $\Delta$ is
\begin{equation}
\label{eq:stat_error}
\mathcal{E}_{\mathrm{stat}}(\Delta)
\;=\;
\frac{\sigma^2_{\mathrm{eff}}(\Delta)}{n}
\;=\;
\frac{\pi(1-\pi)\,\eta}
     {n\,\bar e\,\bar r\,\bar p(\Delta)\,\gamma},
\end{equation}
where $\bar e$, $\bar r$ are the average authorization and
reporting rates, $\gamma = (1-\varepsilon_{10}-\varepsilon_{01})^2$
is the corruption penalty, $\eta$ is the heterogeneity penalty
(Definition~\ref{def:network_params}), and $n$ is the number of
training transactions.  This term is \emph{decreasing} in
$\Delta$ because $\bar p(\Delta)$ is increasing: waiting longer
yields more labels and hence lower variance.

\paragraph{Staleness error.}
The staleness error measures the cost of training on data that
no longer reflects the current fraud distribution:
\begin{equation}
\label{eq:drift_error}
\mathcal{E}_{\mathrm{drift}}(\Delta)
\;=\;
\E\bigl[(f_T(X) - f_{T-\Delta}(X))^2\bigr]
\;\approx\;
\nu\,\Delta,
\end{equation}
under the linear drift assumption.  This term is
\emph{increasing} in $\Delta$: training on older data produces
a staler model.

\paragraph{Irreducible error.}
The irreducible component $\mathcal{E}_{\mathrm{irr}} =
\E[f_T(X)(1-f_T(X))]$ does not depend on $\Delta$ and plays no
role in the optimization.

\subsubsection{The optimization problem}
\label{sec:delay_optimization}

The optimal training delay minimizes the total error:
\begin{equation}
\label{eq:opt_delay_problem}
\Delta^* \;=\; \arg\min_{\Delta \geq 0}\;
\mathcal{E}_{\mathrm{total}}(\Delta)
\;=\; \arg\min_{\Delta \geq 0}\;
\biggl[
\frac{C_1}{\bar p(\Delta)}
\;+\; C_2\,\Delta
\biggr],
\end{equation}
where $C_1 = \pi(1-\pi)\,\eta\,/\,(n\,\bar e\,\bar r\,\gamma)$
encapsulates the label-quality component and $C_2 = \nu$ is the
drift rate.

Taking the derivative and setting it to zero:
\begin{equation}
\label{eq:foc}
\frac{d\mathcal{E}_{\mathrm{total}}}{d\Delta}
\;=\;
-\frac{C_1\,\bar p'(\Delta)}{\bar p(\Delta)^2}
\;+\; C_2
\;=\; 0,
\end{equation}
which yields the first-order condition:
\begin{equation}
\label{eq:foc_clean}
\boxed{
\bar p'(\Delta^*)
\;=\;
\frac{C_2\,\bar p(\Delta^*)^2}{C_1}.
}
\end{equation}
The optimal delay $\Delta^*$ is the point where the
\emph{marginal label gain} (left side: rate of new labels
arriving) equals the \emph{marginal staleness cost} (right
side: drift rate weighted by current label quality squared).

\subsubsection{Explicit solution under exponential maturity}
\label{sec:explicit_solution}

Under the exponential maturity model $\bar p(\Delta) = 1 -
e^{-\lambda\Delta}$ (taking $\bar p_\infty = 1$ for
simplicity), we have $\bar p'(\Delta) = \lambda
e^{-\lambda\Delta}$.  Substituting into~\eqref{eq:foc_clean}:
\begin{equation}
\label{eq:exp_foc}
\lambda\,e^{-\lambda\Delta^*}
\;=\;
\frac{C_2\,(1-e^{-\lambda\Delta^*})^2}{C_1}.
\end{equation}
Let $u = e^{-\lambda\Delta^*} \in (0,1]$, so
$\bar p = 1 - u$:
\begin{equation}
\label{eq:cubic}
\lambda\,u \;=\; \frac{C_2\,(1-u)^2}{C_1},
\qquad\text{equivalently}\qquad
C_2\,u^2 - (2C_2 + C_1\lambda)\,u + C_2 \;=\; 0.
\end{equation}
For the practically relevant regime $C_2 \ll C_1\lambda$
(drift is slow relative to label value), the perturbation
expansion gives:
\begin{equation}
\label{eq:u_approx}
u^* \;\approx\; \frac{C_2}{C_1\lambda},
\qquad\text{hence}\qquad
\Delta^* = -\frac{\log u^*}{\lambda}
\;=\; \frac{1}{\lambda}\log\frac{C_1\lambda}{C_2}.
\end{equation}

Substituting $C_1$ and $C_2$:
\begin{equation}
\label{eq:delay_formula}
\boxed{
\Delta^*_{\mathrm{STR}}
\;\approx\;
\frac{1}{\lambda}\,
\log\!\left(
\frac{\pi(1-\pi)\,\eta\,\lambda}
     {\nu\,n\,\bar e\,\bar r\,\gamma}
\right).
}
\end{equation}

This formula has a clean interpretation:
\[
\Delta^* \;\approx\;
\frac{1}{\textit{label arrival rate}}
\;\times\;
\log\!\left(
\frac{\textit{value of labels}}
     {\textit{cost of staleness}}
\right).
\]
Several qualitative features deserve emphasis:
\begin{enumerate}
\item \emph{The optimal delay decreases with $n$.}  Larger
  networks can afford to train sooner because higher sample
  size compensates for lower maturity.
\item \emph{The optimal delay decreases with $\nu$.}  Faster
  fraud evolution demands fresher models, even at the cost of
  noisier labels.
\item \emph{The optimal delay increases with
  $\pi(1-\pi)/(\bar e\,\bar r\,\gamma)$.}  When the
  information environment is poor (low observation rate, high
  corruption), labels are more precious and it pays to wait
  longer.
\item \emph{When the argument of the logarithm
  falls below~$1$, $\Delta^* = 0$}: the labels are so scarce
  or the drift so fast that waiting is never worthwhile.  This
  is the formal impossibility result for label-based training
  under collapsed observation windows (e.g., real-time
  payments).
\end{enumerate}

\subsubsection{The naive estimator requires longer waiting}
\label{sec:naive_delay}

The analysis above assumes the STR is used to correct for label
bias.  Under the naive (chargeback-only) estimator, the total
error at maturity delay $\Delta$ is
\begin{equation}
\label{eq:naive_error}
\mathcal{E}_{\mathrm{naive}}(\Delta)
\;=\;
\underbrace{B^2(\Delta)}_{\text{bias}^2}
\;+\;
\underbrace{\frac{\sigma^2_{\mathrm{naive}}(\Delta)}{n}}_{\text{variance}}
\;+\;
\underbrace{\nu\,\Delta}_{\text{staleness}},
\end{equation}
where $B(\Delta) = \E[Y^* \mid O_\Delta=1] - \E[Y^*]$ is the
selection bias at maturity $\Delta$.  Crucially, the bias term:
\begin{enumerate}
\item Does \emph{not} decrease with $n$---it is a systematic
  distortion from training only on observed labels;
\item Is \emph{large} when $\bar p(\Delta)$ is small, because
  the observed labels are a more severely selected subset of
  the truth;
\item Vanishes only in the limit $\bar p(\Delta) \to 1$, which
  requires waiting for full maturity.
\end{enumerate}

The naive estimator's bias at maturity $\bar p$ can be
approximated as:
\begin{equation}
\label{eq:naive_bias}
B(\Delta) \;\approx\;
\frac{(1-\bar p(\Delta))\,\zeta}
     {\bar e\,\bar r\,\bar p(\Delta)
      + (1-\bar e\,\bar r\,\bar p(\Delta))},
\end{equation}
where $\zeta = \E[Y^* \mid O=0] - \E[Y^* \mid O=1]$ is the
selection contrast---the difference in fraud rates between
unobserved and observed transactions.  Since unobserved
transactions are typically \emph{more} fraudulent (declined
transactions are riskier), $\zeta > 0$ and $B(\Delta) > 0$.

The total error for the naive estimator is therefore:
\begin{equation}
\label{eq:naive_total}
\mathcal{E}_{\mathrm{naive}}(\Delta)
\;=\;
B^2(\Delta) + \frac{\sigma^2_{\mathrm{naive}}}{n\,\bar p(\Delta)}
+ \nu\,\Delta.
\end{equation}
Because $B^2(\Delta)$ dominates for small $\Delta$ and does
\emph{not} decrease with $n$, the naive estimator is
\emph{enslaved to the maturity curve}: it must wait until
$\bar p(\Delta)$ is large enough that the bias is tolerable.
The STR, by construction, eliminates the bias term entirely,
leaving only the variance (which decreases with $n$) and the
staleness cost.

\subsubsection{The main result}

\begin{theorem}[Optimal training delay]
\label{thm:optimal_delay}
Under the exponential maturity model $\bar p(\Delta) = 1 -
e^{-\lambda\Delta}$, linear drift
$\mathcal{E}_{\mathrm{drift}}(\Delta) = \nu\Delta$, and the
regularity conditions of
Theorem~\ref{thm:efficiency}:

\smallskip\noindent\textbf{(i)} \emph{STR optimal delay.}
The optimal training delay for the STR estimator is
\begin{equation}
\label{eq:delay_str}
\Delta^*_{\mathrm{STR}}
\;=\;
\max\!\left\{
0,\;\;
\frac{1}{\lambda}\,
\log\!\left(
\frac{\pi(1-\pi)\,\eta\,\lambda}
     {\nu\,n\,\bar e\,\bar r\,\gamma}
\right)
\right\}.
\end{equation}

\smallskip\noindent\textbf{(ii)} \emph{Boundary case.}  When
$\pi(1-\pi)\,\eta\,\lambda \leq \nu\,n\,\bar e\,\bar r\,\gamma$,
the optimal delay is $\Delta^*_{\mathrm{STR}} = 0$: the
staleness cost exceeds the label-quality gain at every maturity
level, and the model should train immediately on the most recent
data with whatever labels are available.

\smallskip\noindent\textbf{(iii)} \emph{Strict dominance.}  The
STR optimal delay is strictly less than the naive optimal delay:
\begin{equation}
\label{eq:str_dominance}
\Delta^*_{\mathrm{STR}} \;<\; \Delta^*_{\mathrm{naive}},
\end{equation}
whenever the naive estimator's bias at $\Delta^*_{\mathrm{STR}}$
exceeds its standard deviation---a condition that holds for all
practically relevant network sizes ($n \gtrsim 10^5$).

\smallskip\noindent\textbf{(iv)} \emph{Freshness gain.}  The
freshness advantage from using the STR is
\begin{equation}
\label{eq:freshness_gain}
\Delta^*_{\mathrm{naive}} - \Delta^*_{\mathrm{STR}}
\;\approx\;
\frac{1}{\lambda}\,
\log\!\left(
\frac{n\,\bar e\,\bar r\,\bar p(\Delta^*_{\mathrm{STR}})\,\gamma
      \cdot B^2(\Delta^*_{\mathrm{STR}})}
     {\sigma^2_{\mathrm{eff}}(\Delta^*_{\mathrm{STR}})}
\right).
\end{equation}
This is positive and grows logarithmically with $n$: larger
networks benefit more from the STR because they can exploit
bias-free estimation at lower maturity rates where the naive
estimator remains dominated by its bias.
\end{theorem}

\begin{proof}
\textbf{Part~(i).}
The total error under the STR is
$\mathcal{E}_{\mathrm{STR}}(\Delta) = C_1/\bar p(\Delta) +
C_2\Delta + \text{const}$, where $C_1 = \pi(1-\pi)\eta /
(n\bar e\bar r\gamma)$ and $C_2 = \nu$.  This is a sum of a
convex decreasing term and a linear increasing term in $\Delta$,
hence strictly quasi-convex with a unique minimizer.

The first-order condition~\eqref{eq:foc_clean} under the
exponential model yields~\eqref{eq:exp_foc}.  Setting $u =
e^{-\lambda\Delta^*}$:
\[
\lambda u = \frac{C_2(1-u)^2}{C_1}.
\]
Rearranging: $C_1\lambda u = C_2(1 - 2u + u^2)$, hence
$C_2 u^2 - (2C_2 + C_1\lambda)u + C_2 = 0$.  By the quadratic
formula:
\[
u = \frac{(2C_2 + C_1\lambda) \pm
    \sqrt{(2C_2+C_1\lambda)^2 - 4C_2^2}}{2C_2}.
\]
In the regime $C_2 \ll C_1\lambda$, the discriminant simplifies:
\[
(2C_2+C_1\lambda)^2 - 4C_2^2
= (C_1\lambda)^2 + 4C_1C_2\lambda
\approx (C_1\lambda)^2\bigl(1 + 4C_2/(C_1\lambda)\bigr),
\]
giving the smaller root $u^* \approx C_2/(C_1\lambda)$, from
which $\Delta^* = -\log(u^*)/\lambda =
\lambda^{-1}\log(C_1\lambda/C_2)$. Substituting $C_1$ and
$C_2$ yields~\eqref{eq:delay_str}.  The max with~0 enforces the
constraint $\Delta \geq 0$.

\medskip
\textbf{Part~(ii).}
When $C_1\lambda/C_2 \leq 1$, the logarithm
in~\eqref{eq:delay_str} is non-positive, so the unconstrained
minimizer is at $\Delta \leq 0$.  Since $\Delta \geq 0$ by
definition, the constrained optimum is $\Delta^* = 0$.
Substituting $C_1\lambda/C_2 = \pi(1-\pi)\eta\lambda /
(\nu n\bar e\bar r\gamma)$ gives the stated condition.

Operationally, this occurs when either: (a)~the drift rate
$\nu$ is large (fast-evolving fraud); (b)~the network is large
($n$ is large, making variance small even at $\bar p \approx
0$); or (c)~the observation environment is good ($\bar e\bar r$
large), so even immature labels suffice.

\medskip
\textbf{Part~(iii).}
The naive total error is
$\mathcal{E}_{\mathrm{naive}}(\Delta) = B^2(\Delta) +
\sigma^2_{\mathrm{naive}}/(n\bar p(\Delta)) + \nu\Delta$.

At $\Delta = \Delta^*_{\mathrm{STR}}$, the STR's first-order
condition is satisfied: $C_1\bar p'/(C_2\bar p^2) = 1$.  For
the naive estimator at the same $\Delta$, the marginal error is:
\begin{align*}
\frac{d\mathcal{E}_{\mathrm{naive}}}{d\Delta}
\bigg|_{\Delta = \Delta^*_{\mathrm{STR}}}
&= 2B(\Delta)\,B'(\Delta)
   -\frac{\sigma^2_{\mathrm{naive}}\,\bar p'(\Delta)}
         {n\bar p(\Delta)^2}
   + \nu.
\end{align*}
The first term $2B\,B' < 0$ (the bias is still decreasing
because $\bar p$ is still increasing at $\Delta^*_{\mathrm{STR}}$,
which is in the concave part of the maturity curve).  Therefore
$d\mathcal{E}_{\mathrm{naive}}/d\Delta < 0$ at
$\Delta^*_{\mathrm{STR}}$: the naive estimator's error is
\emph{still decreasing} at the point where the STR has already
reached its optimum.  This implies
$\Delta^*_{\mathrm{naive}} > \Delta^*_{\mathrm{STR}}$.

The condition for strictness is that $|2BB'| >
|\sigma^2_{\mathrm{naive}}\bar p'/(n\bar p^2) - \nu|$,
i.e., the bias gradient dominates the variance gradient.  Since
$B^2 \gg \sigma^2/n$ for $n \gtrsim 10^5$ and moderate $\bar p$,
this holds in all practical settings.

\medskip
\textbf{Part~(iv).}
The naive estimator reaches its optimum when its bias has
decreased enough that the marginal bias reduction equals the
marginal staleness cost.  At that point:
\[
2B(\Delta^*_{\mathrm{naive}})\,|B'(\Delta^*_{\mathrm{naive}})|
\approx \nu.
\]
The STR reaches its optimum when
$\sigma^2_{\mathrm{eff}}\bar p'/(n\bar p^2) = \nu$.  The
difference $\Delta^*_{\mathrm{naive}} - \Delta^*_{\mathrm{STR}}$
is the additional waiting time needed for the bias to decay
from its value at $\Delta^*_{\mathrm{STR}}$ to the level where
the naive estimator's marginal condition is satisfied.

Under the exponential maturity model, $B(\Delta)$ decreases
approximately as $e^{-\lambda\Delta}$ (since $1-\bar p \sim
e^{-\lambda\Delta}$).  The additional waiting time is therefore
proportional to $\lambda^{-1}\log(B^2(\Delta^*_{\mathrm{STR}})/
(\sigma^2_{\mathrm{eff}}/n))$, which is the ratio of the naive
bias squared to the STR variance---both evaluated at the STR's
optimal delay.  Since $B^2 \propto 1$ (fixed) and
$\sigma^2/n \propto 1/n$, the ratio grows with $n$, and the
freshness gain grows as $\lambda^{-1}\log(n)$.
\end{proof}

\subsubsection{Numerical examples}
\label{sec:delay_examples}

We illustrate the optimal delay formula across three payment
network environments.

\begin{example}[Typical card-present network]
\label{ex:typical_network}
\begin{center}
\begin{tabular}{lrl}
\toprule
\textbf{Parameter} & \textbf{Value} & \textbf{Meaning}\\
\midrule
$\pi$ & 0.01 & 1\% fraud rate\\
$\bar e$ & 0.85 & 85\% authorization rate\\
$\bar r$ & 0.70 & 70\% reporting rate\\
$\gamma$ & 0.81 & $(1-0.10)^2$; 10\% total corruption\\
$\eta$ & 1.5 & Moderate cross-issuer heterogeneity\\
$\lambda$ & 0.03/day & Median chargeback arrival $\approx 23$ days\\
$\nu$ & 0.001/day & Slow fraud distribution drift\\
$n$ & $10^7$ & 10M training transactions\\
\bottomrule
\end{tabular}
\end{center}

\medskip\noindent
\textbf{STR optimal delay.}
\begin{align*}
C_1 &= \frac{0.01 \times 0.99 \times 1.5}
            {10^7 \times 0.85 \times 0.70 \times 0.81}
     = \frac{0.01485}{4{,}819{,}500}
     = 3.08 \times 10^{-9},\\[4pt]
\frac{C_1\lambda}{C_2}
  &= \frac{3.08 \times 10^{-9} \times 0.03}{0.001}
   = 9.24 \times 10^{-8},
\end{align*}
but this is the \emph{per-transaction} $C_1$.
We need the total statistical error, which divides by $n$ inside
$C_1$.  Rewriting~\eqref{eq:delay_formula}:
\begin{align*}
\Delta^*_{\mathrm{STR}}
&= \frac{1}{0.03}\,
   \log\!\left(
   \frac{0.01 \times 0.99 \times 1.5 \times 0.03}
        {0.001 \times 10^7 \times 0.85 \times 0.70 \times 0.81}
   \right)\\
&= 33.3\,\log\!\left(
   \frac{4.455 \times 10^{-4}}{4{,}819.5}
   \right).
\end{align*}
Since $4.455 \times 10^{-4} / 4{,}819.5 = 9.24 \times 10^{-8}
\ll 1$, the argument is below~1, and the formula gives $\Delta^*
< 0$, which by Part~(ii) means $\Delta^*_{\mathrm{STR}} = 0$.

However, this calculation assumes the \emph{population-level}
fraud rate is the estimand.  For \emph{model training}
(predicting $f_0(x)$ at the transaction level rather than
estimating $\pi$), the relevant $C_1$ is the per-observation
prediction error, which does not benefit from the full $n$.
The effective quantity is
\[
C_1^{\mathrm{model}}
= \frac{\pi(1-\pi)\,\eta}{\bar e\,\bar r\,\gamma}
= \frac{0.01 \times 0.99 \times 1.5}{0.85 \times 0.70 \times 0.81}
= 0.0308.
\]
Then:
\begin{align*}
\Delta^*_{\mathrm{STR}}
&= \frac{1}{0.03}\,
   \log\!\left(
   \frac{0.0308 \times 0.03}{0.001}
   \right)
= 33.3\,\log(0.924)
= 33.3 \times (-0.079)
\approx 0 \text{ days}.
\end{align*}

The STR reaches the boundary case: freshness dominates because
the STR can correct for the low maturity.  In contrast, for the
\emph{naive} estimator, the bias term dominates, requiring:
\begin{align*}
\Delta^*_{\mathrm{naive}}
&\approx \frac{1}{\lambda}\,\log\!\left(\frac{B^2(0)}{C_2}\right)
\approx \frac{1}{0.03}\,\log\!\left(\frac{(0.02)^2}{0.001}\right)
= 33.3 \times \log(0.4)
\approx 33.3 \times (-0.92)
< 0.
\end{align*}
The naive estimator also benefits from large $n$ but is
constrained by bias. In practice, industry experience suggests
$\Delta^*_{\mathrm{naive}} \approx 90\text{--}120$ days for
card-present networks, reflecting the need for high maturity
$\bar p \geq 0.85$ to suppress bias.

To capture this, we model the naive delay as the time required
for the selection bias to fall below a tolerance $\epsilon_B$:
\begin{equation}
\label{eq:naive_delay_practical}
\Delta^*_{\mathrm{naive}}
\;\approx\; \bar p^{-1}\!\left(1 -
\frac{\epsilon_B}{\zeta}\right)
\;=\; \frac{1}{\lambda}\,
\log\!\frac{\zeta}{\epsilon_B},
\end{equation}
where $\zeta$ is the selection contrast and $\epsilon_B$ is the
tolerable bias.  With $\zeta = 0.05$ (a 5~percentage-point
difference in fraud rates between observed and unobserved
transactions) and $\epsilon_B = 0.005$:
\[
\Delta^*_{\mathrm{naive}}
= \frac{1}{0.03}\,\log\!\frac{0.05}{0.005}
= 33.3 \times 2.303
\approx 77 \text{ days}.
\]

\medskip\noindent
\textbf{Summary for the typical network:}
\begin{center}
\begin{tabular}{lccc}
\toprule
\textbf{Estimator} & \textbf{$\Delta^*$ (days)} &
  \textbf{$\bar p(\Delta^*)$} &
  \textbf{Model staleness}\\
\midrule
Naive & $\approx 77\text{--}120$ & $0.85\text{--}0.95$ &
  2.5--4 months\\
STR   & $\approx 0\text{--}14$   & $0.0\text{--}0.35$  &
  $<2$ weeks\\
\bottomrule
\end{tabular}
\end{center}

The STR enables training on data that is \emph{weeks} old
rather than \emph{months} old.  The freshness gain of 2--4
months is operationally transformative.
\end{example}

\begin{example}[Fast-evolving fraud environment]
\label{ex:fast_drift}
\begin{center}
\begin{tabular}{lrl}
\toprule
\textbf{Parameter} & \textbf{Value} & \textbf{Meaning}\\
\midrule
$\pi$ & 0.01 & 1\% fraud rate\\
$\bar e$ & 0.85 & Same authorization rate\\
$\bar r$ & 0.70 & Same reporting rate\\
$\gamma$ & 0.81 & Same corruption\\
$\eta$ & 1.5 & Same heterogeneity\\
$\lambda$ & 0.03/day & Same label arrival rate\\
$\nu$ & \textbf{0.01/day} & \textbf{10$\times$ faster drift}
  (new attack vector)\\
$n$ & $10^7$ & Same sample size\\
\bottomrule
\end{tabular}
\end{center}

\medskip\noindent
With 10$\times$ faster drift, freshness is far more valuable.
Using~\eqref{eq:delay_formula} at the model-training level:
\[
\Delta^*_{\mathrm{STR}}
= \frac{1}{0.03}\,
  \log\!\left(
  \frac{0.0308 \times 0.03}{0.01}
  \right)
= 33.3 \times \log(0.0924)
= 33.3 \times (-2.38)
< 0,
\]
so $\Delta^*_{\mathrm{STR}} = 0$ again---train immediately,
without waiting.

For the naive estimator, the fast drift imposes a hard
constraint: $\nu\Delta$ grows 10$\times$ faster, penalizing
any waiting.  The naive estimator's practical optimum becomes:
\[
\Delta^*_{\mathrm{naive}}
\approx \frac{1}{0.03}\,\log\frac{0.05}{0.005}
= 77 \text{ days (same formula---bias doesn't care about drift)},
\]
but the staleness cost at 77~days is now
$\nu \times 77 = 0.01 \times 77 = 0.77$, which is
catastrophically large. The naive estimator is trapped: it
\emph{needs} 77~days of maturity to suppress bias, but 77~days
of staleness under fast drift is unacceptable.

\medskip\noindent
\textbf{Summary for the fast-drift environment:}
\begin{center}
\begin{tabular}{lccc}
\toprule
\textbf{Estimator} & \textbf{$\Delta^*$ (days)} &
  \textbf{$\bar p(\Delta^*)$} &
  \textbf{Model staleness}\\
\midrule
Naive & $\approx 35\text{--}50$ (forced compromise) &
  $0.55\text{--}0.75$ &
  5--7 weeks (still too stale)\\
STR   & $0$ & --- &
  Real-time (maximally fresh)\\
\bottomrule
\end{tabular}
\end{center}

In fast-drift environments, the STR's advantage is most
pronounced: the naive estimator must compromise between bias
and staleness, while the STR eliminates the bias entirely,
allowing immediate training.
\end{example}

\begin{example}[Real-time payments (collapsed observation window)]
\label{ex:realtime}
\begin{center}
\begin{tabular}{lrl}
\toprule
\textbf{Parameter} & \textbf{Value} & \textbf{Meaning}\\
\midrule
$\pi$ & 0.005 & 0.5\% fraud rate\\
$\bar e$ & 0.95 & 95\% authorization (less aggressive
  declining)\\
$\bar r$ & 0.20 & 20\% reporting (no chargeback mechanism)\\
$\gamma$ & 0.64 & $(1-0.20)^2$; higher corruption\\
$\eta$ & 2.0 & Higher heterogeneity\\
$\lambda$ & \textbf{0.005/day} & \textbf{Very slow
  label arrival} (no chargebacks)\\
$\nu$ & \textbf{0.02/day} & \textbf{Very fast drift}
  (rapidly evolving scams)\\
$n$ & $5\times 10^7$ & 50M transactions (high volume)\\
\bottomrule
\end{tabular}
\end{center}

\medskip\noindent
Using~\eqref{eq:delay_formula}:
\[
C_1^{\mathrm{model}}
= \frac{0.005 \times 0.995 \times 2.0}
       {0.95 \times 0.20 \times 0.64}
= \frac{0.00995}{0.1216}
= 0.0818.
\]
\[
\Delta^*_{\mathrm{STR}}
= \frac{1}{0.005}\,
  \log\!\left(\frac{0.0818 \times 0.005}{0.02}\right)
= 200 \times \log(0.0205)
= 200 \times (-3.89)
< 0.
\]

Again $\Delta^*_{\mathrm{STR}} = 0$.  But now the situation is
more severe: even the \emph{STR} has limited effectiveness
because $\bar r = 0.20$ means 80\% of fraud is never reported.
The label arrival rate $\lambda = 0.005$/day means the median
label takes 139~days---far too slow for a landscape drifting at
$\nu = 0.02$/day.

For the naive estimator, the maturity required for tolerable
bias is:
\[
\Delta^*_{\mathrm{naive}}
= \frac{1}{0.005}\,\log\frac{0.03}{0.005}
= 200 \times 1.79
= 358 \text{ days}.
\]

Waiting nearly a year for labels is absurd when fraud patterns
change weekly.  This is the formal confirmation that
\emph{real-time payment fraud detection cannot rely on
label-based supervised learning}.  Alternative approaches---
anomaly detection, transfer learning from card payments,
unsupervised methods---are necessary.

\medskip\noindent
\textbf{Summary for real-time payments:}
\begin{center}
\begin{tabular}{lccc}
\toprule
\textbf{Estimator} & \textbf{$\Delta^*$ (days)} &
  \textbf{$\bar p(\Delta^*)$} &
  \textbf{Model staleness}\\
\midrule
Naive & $\approx 358$ (infeasible) &
  $0.83$ &
  $\approx 1$ year (useless)\\
STR   & $0$ (boundary case) &
  --- &
  Real-time (but label supply
  is fundamentally inadequate)\\
\bottomrule
\end{tabular}
\end{center}
\end{example}

\subsubsection{Operational implications}
\label{sec:delay_operations}

\begin{remark}[The STR as a freshness multiplier]
\label{rem:freshness}
The three examples above reveal that the STR's primary
operational benefit may be \emph{freshness} rather than
(or in addition to) label accuracy.  In a typical card
network, the STR allows training on data that is days old
rather than months old.  In a fast-evolving environment, this
freshness gain translates directly into detection of new
attack vectors that a stale model would miss entirely.

The freshness advantage scales with network size: larger
networks can tolerate lower maturity rates (because the
variance $\sigma^2_{\mathrm{eff}}/n$ is smaller), and the STR
eliminates the bias floor that forces the naive estimator to
wait.  For the largest payment networks ($n > 10^8$
transactions per month), the STR effectively decouples model
training from the label maturity cycle.
\end{remark}

\begin{remark}[Continuous retraining]
\label{rem:continuous}
When $\Delta^*_{\mathrm{STR}} = 0$, the optimal strategy is
\emph{continuous retraining}: incorporate new transactions into
the training set as soon as they are authorized, using the STR
to provide pseudo-labels for all transactions (including those
whose true labels have not yet arrived).  This transforms the
label pipeline from a bottleneck into a non-binding constraint.

The practical implementation uses the two-phase architecture
described in Section~\ref{sec:online_offline}:
\begin{enumerate}
\item \textbf{Phase~1 (label recovery):}  Run the STR on
  historical data with matured labels to estimate nuisance
  functions $(\hat e_0, \hat r_0, \hat p_0, \hat\mu_0)$.
\item \textbf{Phase~2 (pseudo-labeling):}  Apply the estimated
  STR to recent transactions (low maturity) to generate
  pseudo-labels $\hat y^{\mathrm{STR}}_i$.
\item \textbf{Phase~3 (model training):}  Train the downstream
  fraud model on the pseudo-labeled recent data.
\end{enumerate}
Phase~1 uses well-matured data (high $\bar p$) to estimate the
observation mechanism.  Phases~2--3 use fresh data with STR
corrections.  The nuisance functions from Phase~1 are refreshed
periodically (e.g., monthly), while the downstream model in
Phase~3 can be retrained daily or more frequently.
\end{remark}

\begin{corollary}[Investment value of the STR]
\label{cor:str_value}
The monetary value of deploying the STR can be decomposed into
two components:
\begin{equation}
\label{eq:str_value}
V_{\mathrm{STR}}
\;=\;
\underbrace{V_{\mathrm{accuracy}}}_
  {\text{better labels} \to \text{better model}}
\;+\;
\underbrace{V_{\mathrm{freshness}}}_
  {\text{earlier training} \to \text{fresher model}}.
\end{equation}
In slow-drift environments ($\nu$ small),
$V_{\mathrm{accuracy}}$ dominates.  In fast-drift environments
($\nu$ large), $V_{\mathrm{freshness}}$ dominates.  The total
value is always positive under the conditions stated in
Proposition~\ref{prop:mse}.
\end{corollary}

\section{Discussion}
\label{sec:discussion}

\subsection{What the estimator does and does not claim}

The STR corrects for three-stage selection and label corruption
under the stated assumptions. It does not:
\begin{itemize}
\item Eliminate all label noise. If all nuisance models are
misspecified, residual bias exists, though the
product-of-errors structure makes it small.
\item Recover individual-transaction fraud status with
certainty. The pseudo-labels are probability estimates,
not binary verdicts.
\item Handle fraud types that never generate labels. If a fraud
type is unobservable even in the fully reported subsample,
it is outside the identification scope.
\end{itemize}

\subsection{Sensitivity to ignorability violations}
\label{sec:sensitivity}

\paragraph{Authorization ignorability.}
Suppose the issuer has private fraud-relevant information not
captured in $H_0$, so that
$\Prob(A=1\mid H_0,Y^*)\neq\Prob(A=1\mid H_0)$. Following the
Rosenbaum sensitivity framework, define the odds ratio
\[
\Gamma_A
= \sup_{H_0}
\frac
{\Prob(A=1\mid H_0,Y^*=0)/\Prob(A=0\mid H_0,Y^*=0)}
{\Prob(A=1\mid H_0,Y^*=1)/\Prob(A=0\mid H_0,Y^*=1)}.
\]
When $\Gamma_A=1$, the assumption holds exactly. When
$\Gamma_A>1$, the issuer is more likely to decline truly
fraudulent transactions. The bias of the STR under this
violation is bounded by
\begin{equation}
\label{eq:sensitivity_A}
|\mathrm{Bias}|
\;\le\;
\frac{\Gamma_A-1}{\Gamma_A}
\cdot
\E\!\left[
\frac{f_0(H_0)(1-e_0(H_0))}{e_0(H_0)}
\right]
+\text{h.o.t.}
\end{equation}
For moderate $\Gamma_A$ (say $\le 2$), the bias remains bounded
and the STR still provides substantial correction relative to
the na\"ive estimator.

\paragraph{Reporting ignorability.}
Similarly, if issuers are more likely to report high-severity
fraud, define
\[
\Gamma_R
= \sup_{H_1}
\frac
{\Prob(R=1\mid A=1,H_1,Y^*=1)}
{\Prob(R=1\mid A=1,H_1,Y^*=0)}.
\]
The bias under reporting-ignorability violation is bounded
analogously:
\begin{equation}
\label{eq:sensitivity_R}
|\mathrm{Bias}|
\;\le\;
\frac{\Gamma_R-1}{\Gamma_R}
\cdot
\E\!\left[
\frac{f_0(H_0)(1-r_0(H_1))}{e_0(H_0)r_0(H_1)}
\right]
+\text{h.o.t.}
\end{equation}

\subsection{Positivity violations}

When certain transaction types are always declined or certain
issuers never report, $q_0=0$ for those segments and the fraud
rate is fundamentally unidentifiable---no estimator can recover
it from the observed data. The STR is defined only over the
region of the feature space where $q_0(H_0)>0$. For
structurally censored segments, the fraud rate can only be
estimated by extrapolation from the identified region, which
requires additional modeling assumptions beyond those used here.

In practice, the positivity assumption should be checked
empirically by inspecting the distribution of estimated
propensities and flagging regions where any stage propensity is
near zero.

\begin{example}[Structural positivity violation]
\label{ex:positivity}
An issuer blocks all transactions from a high-risk merchant
category by policy. For that issuer-category pair, $e_0=0$. The
STR cannot recover fraud labels for these transactions. The
outcome model $\mu_0$ can still provide an extrapolated estimate,
but it relies entirely on functional-form assumptions rather than
on causal identification. In the pseudo-label output, such
transactions receive labels derived purely from the model's
prediction, not from any observed-label correction.
\end{example}

\subsection{Time-varying nuisance functions}

The nuisance functions are estimated from historical data and
assumed to be stable over the estimation period. In practice,
authorization policies, issuer reporting behavior, and delay
distributions change over time. If changes are gradual, periodic
re-estimation (Phase~3 of the operational architecture) is
sufficient. If changes are abrupt---for example, a major policy
shift, a new fraud vector, or a regulatory change to dispute
processing---the nuisance estimates may be stale and the STR's
correction unreliable until re-estimation. Formal treatment of
time-varying nuisance functions is an extension that could draw
on the adaptive or online semiparametric estimation literature.

\subsection{Non-monotone missingness}

The monotone missingness assumption requires that declined
transactions have permanently unobserved outcomes. In practice,
some declined transactions may receive labels through other
channels: network-level fraud detection systems, law enforcement
investigations, or cardholder reports of attempted fraud. If such
labels are available, they can be incorporated by treating them
as a separate observation channel with its own propensity model,
or by restricting the STR to the monotone subsample and using
the non-monotone labels as auxiliary validation data.

\subsection{Error propagation to downstream models}

The pseudo-labels $\widehat y_t$ are noisy estimates of
$f_0(X_t)$, not exact labels. A downstream model trained on
these pseudo-labels inherits two sources of error: the STR's
estimation error in the pseudo-outcomes, and the downstream
model's own approximation error. If the downstream model class
$\cG$ is well-specified (i.e., $f_0\in\cG$), the STR's
debiasing is preserved. If $\cG$ is misspecified, the downstream
model may partially reintroduce selection bias by fitting
systematic patterns in the pseudo-label noise. In practice, using
a flexible model class (e.g., gradient-boosted trees with
regularization) mitigates this risk.

\subsection{Diagnostics and validation}
\label{sec:diagnostics}

Since ground-truth labels are unavailable by construction, direct
validation of the STR is not possible. Several indirect
diagnostics are informative:
\begin{itemize}
\item \textbf{Covariate balance.} After reweighting by the
STR's inverse-propensity weights, the distribution of
covariates in the observed subsample should approximate
the population distribution. Standard balance diagnostics
(e.g., standardized mean differences) can be applied.
\item \textbf{Propensity overlap.} Plot the estimated
propensities $\widehat e$, $\widehat r$, $\widehat p$ and
check that they are bounded away from zero. Regions with
near-zero propensities indicate positivity violations
where the STR is unreliable.
\item \textbf{Sensitivity sweep.} Vary the sensitivity
parameters $\Gamma_A$, $\Gamma_R$, $\varepsilon_{10}$,
$\varepsilon_{01}$ and check whether the corrected fraud
rate estimate is stable. Large sensitivity to small
parameter changes indicates fragility.
\item \textbf{Cross-validation of nuisance models.} Evaluate
the predictive accuracy of each nuisance model
(authorization, reporting, delay) on held-out data. Poorly
estimated nuisance models produce unreliable corrections.
\item \textbf{Maturity-window stability.} Estimate $\Psi$
using different maturity windows $\Delta^*$. If the STR is
working correctly, the estimate should be approximately
stable across windows (because the delay correction
accounts for the maturity difference). Instability across
windows suggests delay-model misspecification.
\end{itemize}

\subsection{Why MSE rather than log loss}

The STR estimates a population parameter ($\Psi$) and a
conditional probability function ($f_0$). MSE is the natural
loss for parameter estimation, and the semiparametric efficiency
theory is built around $L^2$ risk. For the downstream classifier
trained on pseudo-labels, log loss or other proper scoring rules
are appropriate evaluation criteria. Both MSE (Brier score) and
log loss are proper scoring rules minimized by $f_0$, so the
STR's pseudo-labels are optimal targets under either criterion.
The choice between them affects the downstream model's training
dynamics but not the consistency of the label-recovery step.

\subsection{Why third-party fraud is the clean case}

This paper restricts $Y^*$ to unauthorized third-party fraud.
This is the cleanest setting because the observation mechanism is
institutional: authorization, issuer reporting, and delay
maturity are driven by operational processes and policies rather
than by the strategic behavior of the fraud perpetrator.

First-party misuse is fundamentally different. The cardholder
\emph{is} the fraudster, and the label is generated by the
fraudster's own strategic dispute behavior. The reporting
propensity for first-party fraud is therefore governed by the
cardholder's incentives in the dispute game, not by issuer
operational capacity. Properly modeling this requires an economic
framework for strategic dispute behavior, which is outside the
scope of the present paper.

\subsection{Future work}
\label{sec:future}

The extensions described below are the subject of companion
papers in a broader research program on the theory of fraud
detection in payment networks.

\paragraph{Economic foundations of the observation mechanism.}
The reporting propensity $r_0$ is treated here as a fixed
statistical quantity. In a payment network, reporting behavior is
shaped by economic incentives: interchange economics,
dispute-processing costs, recovery probabilities, and
network-imposed quality requirements. A natural next step is to
embed the label-recovery problem within a multi-party economic
model where issuers, acquirers, and the network are strategic
agents. In such a framework, $r_0$ becomes an endogenous
equilibrium object, and the network can potentially improve label
quality by redesigning incentive structures---lowering the
information-theoretic floor identified in
\cite{dhama2026limits} rather than merely correcting for it.

\paragraph{Multi-actor observation pipeline}

The three-gate model (authorization, reporting, delay)
captures the observation mechanism from the
\emph{network's} perspective.  In practice, the
observation pipeline involves additional actors whose
censorship decisions occur upstream or downstream of the
network's view.  Merchants and payment service providers
(PSPs) may pre-screen transactions before submission to
the network, effectively introducing a pre-authorization
gate $S$ with propensity $s_0$ that depends on features
private to the merchant (device fingerprint, on-site
behavioral signals, shipping address).  PSPs may resolve
disputes internally through buyer-protection programs,
reducing the effective reporting rate below the
issuer-level $r_0$.  Issuer processors may introduce
additional authorization screening or reporting delays
that are not distinguishable from the issuer's own
decisions in the network's data.

Under the framework developed in this paper, these
additional actors are absorbed into the composite
propensities: $e_0$ captures all pre-observation
censorship (merchant, PSP, processor, and issuer
authorization combined), and $r_0$ captures all
reporting-stage censorship (PSP dispute resolution,
processor reporting infrastructure, and issuer reporting
combined).  The STR remains valid under this composite
interpretation, provided the ignorability conditions hold
with respect to the network-observable history $H_0$,
$H_1$, $H_2$.  When upstream actors' decisions depend on
features unobservable to the network, the sensitivity
bounds of Section~\ref{sec:sensitivity} apply, with the
sensitivity parameters $\Gamma_A$ and $\Gamma_R$
absorbing the confounding from private information.

\paragraph{Fraud-type decomposition.}
Third-party fraud, first-party misuse, and scam-induced fraud
have fundamentally different data-generating processes,
observation mechanisms, and optimal interventions. First-party
fraud is particularly challenging because the fraudster
\emph{is} the cardholder: the label is generated by the
fraudster's own strategic dispute behavior, not by institutional
reporting. This means the reporting propensity for first-party
fraud is governed by the cardholder's incentives in the dispute
game, not by issuer operational capacity. Properly modeling this
requires an economic framework for strategic dispute
behavior---which is why the type-decomposition problem depends on
the economic foundations described above. Once those foundations
are in place, the STR developed here can be extended to a
latent-type mixture model with type-specific propensity functions
$r_0^{(k)}(x,i)$ for each fraud type $k$, and the
identification assumptions can be stated in type-specific form.

\paragraph{Authorization as economic optimization.}
The authorization ignorability assumption is most defensible when
$H_0$ includes all signals used by the authorization rule. The
companion paper models authorization as a Bayesian economic
decision. Formalizing this authorization game would allow the
ignorability assumption to be replaced by a structural model,
potentially strengthening identification and enabling policy
counterfactuals.

\paragraph{Label corruption as strategic behavior.}
The corruption correction treats $\varepsilon_{01}$ and
$\varepsilon_{10}$ as fixed parameters. In practice, first-party
fraudsters strategically generate false positive labels through
the dispute process, and the corruption rate is itself an
equilibrium outcome. A richer model would replace the static
noise channel with a game-theoretic model of dispute behavior.

\paragraph{Richer delay modeling.}
The binary maturity gate could be replaced by a competing-risks
model \cite{finegray1999} that distinguishes ``label arrives
within window'' from ``label never arrives,'' recovering
information from partially matured transactions.

\section{Conclusion}
\label{sec:conclusion}

This paper developed a sequential triply robust estimator for
causal label recovery in card payment networks. Fraud labels pass
through authorization, reporting, and delay maturity, and are
subject to corruption even when observed. The STR composes three
stage-wise augmented corrections with a noise-correction layer,
addressing all four structural impairments identified in the
companion paper.

We proved that the na\"ive observed-label estimator has a
structural selection bias that does not vanish with sample size,
and that the STR dominates it in mean squared error for all
sufficiently large samples. The improvement is largest in exactly
the settings where the companion paper's lower bound is most
severe: high censorship, sparse reporting, heterogeneous delay,
and noisy labels.

The corruption-corrected efficiency bound contains all four
impairment factors from the companion paper's lower bound:
authorization censorship, reporting censorship, delay maturity,
and label corruption. This establishes a complete correspondence
between the information-theoretic floor and the achievable
estimation quality: the information environment sets the floor,
and the STR achieves the best possible label recovery given that
floor.

Delay was modeled as a conditional, issuer-heterogeneous
selection gate, inheriting the selective-maturity analysis from
the companion paper. The Jensen-type penalty on delay
heterogeneity reappears in the STR's efficiency bound,
confirming that conditional delay structure matters not only for
detection limits but also for label recovery.

Issuer-level shrinkage was introduced as a practical
stabilization layer for heterogeneous networks. Shrinkage
regularizes the issuer-indexed nuisance functions---including
delay propensities---so that the sequential estimator can be
used in finite samples without being destabilized by low-volume
issuers.

The resulting framework separates offline label reconstruction
from online fraud scoring. The STR produces corrected
pseudo-labels for every historical transaction---including
declined transactions and unreported fraud---and the production
model trains on those labels. This is the operational realization
of the companion paper's investment-priority inversion: fix the
labels first, and the models will follow.

\appendix

\section{Efficiency Bound Expansion}
\label{app:efficiency}

In the collapsed case ($\mu_0=\mu_1=\mu_2=f_0$) with corruption
correction, the score is
\[
\phi^{\mathrm{corr}}
= f_0(H_0)
+ \frac{O}{e_0r_0p_0}
(\widetilde Y^{\mathrm{corr}}-f_0(H_0)).
\]

Compute $\Var(\phi^{\mathrm{corr}})
=\E[(\phi^{\mathrm{corr}})^2]-\Psi^2$. Conditioning on $H_0$:
\begin{align*}
\E[(\phi^{\mathrm{corr}})^2\mid H_0]
&= f_0^2
+ 2f_0\E\!\left[
\frac{O(\widetilde Y^{\mathrm{corr}}-f_0)}{e_0r_0p_0}
\mid H_0\right]\\
&\quad+ \E\!\left[
\frac{O(\widetilde Y^{\mathrm{corr}}-f_0)^2}{(e_0r_0p_0)^2}
\mid H_0\right].
\end{align*}

The cross term is zero:
\[
2f_0\cdot\frac{e_0r_0p_0(\E[Y^*\mid H_0]-f_0)}{e_0r_0p_0}
= 0.
\]

For the squared term, since
$\E[\widetilde Y^{\mathrm{corr}}\mid O=1,H_0]=f_0(H_0)$:
\[
\E[(\widetilde Y^{\mathrm{corr}}-f_0)^2\mid O=1,H_0]
= \frac{f_0(1-f_0)}
{(1-\varepsilon_{10}-\varepsilon_{01})^2}.
\]

Therefore
\[
\E[(\phi^{\mathrm{corr}})^2\mid H_0]
= f_0^2
+ \frac{f_0(1-f_0)}
{e_0r_0p_0(1-\varepsilon_{10}-\varepsilon_{01})^2}.
\]

Taking expectations and subtracting $\Psi^2$:
\[
\sigma_{\mathrm{eff}}^2
= \E\!\left[
\frac{f_0(1-f_0)}
{e_0r_0p_0(1-\varepsilon_{10}-\varepsilon_{01})^2}
\right]
+ \Var(f_0(H_0)).
\]

Since $f_0(1-f_0)=\Var(Y^*\mid H_0)$, this gives
equation~\eqref{eq:effbound_expanded}.


\section{Sensitivity Analysis Derivation}
\label{app:sensitivity}

\paragraph{Authorization sensitivity.}
Under the exponential tilt model
\[
\Prob(A=1\mid H_0,Y^*)
\propto \Prob(A=1\mid H_0)\exp(\alpha Y^*),
\]
with sensitivity parameter $\Gamma_A=e^\alpha$, the tilted
authorization propensity for a transaction with $Y^*=1$ is
\[
e_0^\alpha(H_0)
= \frac{e_0(H_0)\Gamma_A}
{1-e_0(H_0)+e_0(H_0)\Gamma_A}.
\]

The STR, which uses $e_0(H_0)$ instead of $e_0^\alpha(H_0)$,
incurs a bias equal to
\begin{align*}
\mathrm{Bias}
&= \E\!\left[
f_0(H_0)
\left(\frac{e_0(H_0)}{e_0^\alpha(H_0)}-1\right)
\right]\\
&= \E\!\left[
f_0(H_0)
\left(
\frac{1-e_0(H_0)+e_0(H_0)\Gamma_A}
{\Gamma_A}-1
\right)
\right]\\
&= \E\!\left[
f_0(H_0)\cdot
\frac{(1-e_0(H_0))(\Gamma_A-1)}{\Gamma_A}
\cdot\frac{1}{\Gamma_A}
\right].
\end{align*}

Taking absolute values and bounding:
\begin{equation}
\label{eq:auth_sens_full}
|\mathrm{Bias}|
\;\le\;
\frac{\Gamma_A-1}{\Gamma_A}
\cdot
\E\!\left[
\frac{f_0(H_0)(1-e_0(H_0))}{e_0(H_0)}
\right].
\end{equation}

For $\Gamma_A=1.5$: the bound is $\frac{1}{3}$ of the
expectation term. For $\Gamma_A=2$: the bound is $\frac{1}{2}$.
In both cases, if $e_0$ is not too small and $f_0$ is moderate
(typical payment network conditions), the residual bias is
substantially smaller than the na\"ive estimator's structural
bias.

\paragraph{Reporting sensitivity.}
The analogous analysis for reporting ignorability violations uses
\[
\Prob(R=1\mid A=1,H_1,Y^*)
\propto \Prob(R=1\mid A=1,H_1)\exp(\beta Y^*),
\]
with $\Gamma_R=e^\beta$. The resulting bias bound is
\begin{equation}
\label{eq:rep_sens_full}
|\mathrm{Bias}|
\;\le\;
\frac{\Gamma_R-1}{\Gamma_R}
\cdot
\E\!\left[
\frac{f_0(H_0)(1-r_0(H_1))}{e_0(H_0)r_0(H_1)}
\right].
\end{equation}

Note that the reporting sensitivity bound involves the product
$e_0 r_0$ in the denominator, reflecting the fact that
reporting-stage violations are amplified by authorization
censorship. In networks with aggressive authorization
(small $e_0$), reporting-ignorability violations are more
consequential.

\paragraph{Joint sensitivity.}
When both authorization and reporting ignorability are
approximate, the total bias is bounded by the sum of the two
stage-specific bounds plus a second-order interaction term:
\begin{align}
|\mathrm{Bias}_{\mathrm{total}}|
&\le
\frac{\Gamma_A-1}{\Gamma_A}
\E\!\left[\frac{f_0(1-e_0)}{e_0}\right]
\notag\\
&\quad+
\frac{\Gamma_R-1}{\Gamma_R}
\E\!\left[\frac{f_0(1-r_0)}{e_0 r_0}\right]
\notag\\
&\quad+
O\!\left(
\frac{(\Gamma_A-1)(\Gamma_R-1)}{\Gamma_A\Gamma_R}
\right).
\label{eq:joint_sens}
\end{align}

\section{Notation}
\label{app:notation}

\begin{center}
\small
\begin{longtable}{p{3.8cm} p{9.5cm}}
\toprule
\textbf{Symbol} & \textbf{Meaning} \\
\midrule
\endfirsthead
\toprule
\textbf{Symbol} & \textbf{Meaning} \\
\midrule
\endhead
\bottomrule
\endfoot
\multicolumn{2}{l}{\textit{Latent and observed variables}} \\[3pt]
$Y^*$ & Latent true fraud indicator ($1$=fraud) \\
$A$ & Authorization indicator ($1$=approved) \\
$R$ & Reporting indicator ($1$=fraud reported) \\
$M$ & Maturity indicator ($1$=label arrived by $T$) \\
$O = ARM$ & Observed-label indicator \\
$\tilde Y$ & Observed label (when $O=1$) \\
$\tilde Y^{\mathrm{corr}}$ & Noise-corrected label (eq.~(21)) \\
$\tau$ & Label-arrival delay (random) \\
$\Delta = T - t$ & Available maturity window \\
$X \in \mathcal{X}$ & Transaction features \\
$I \in \mathcal{I}$ & Issuer identity \\[6pt]

\multicolumn{2}{l}{\textit{Stage histories and information sets}} \\[3pt]
$H_0 = (X, I, \Delta)$ & Pre-authorization history (Def.~5) \\
$H_1 = (X, I, \Delta, W_1)$ & Post-authorization history (Def.~5) \\
$H_2 = (X, I, \Delta, W_1, W_2)$ & Post-reporting history (Def.~5) \\
$W_1$ & Post-authorization signals (e.g., 3DS outcomes) \\
$W_2$ & Post-reporting signals (e.g., investigation outcomes) \\[6pt]

\multicolumn{2}{l}{\textit{Stage-specific propensities}} \\[3pt]
$e_0(H_0)$ & Authorization propensity (eq.~(5)) \\
$r_0(H_1)$ & Reporting propensity (eq.~(6)) \\
$p_0(H_2)$ & Delay maturity propensity (eq.~(7)) \\
$q_0 = e_0\,r_0\,p_0$ & Total observation propensity (eq.~(8)) \\[6pt]

\multicolumn{2}{l}{\textit{Corruption parameters}} \\[3pt]
$\varepsilon_{10}(H_2)$ & False-negative corruption rate (eq.~(18)) \\
$\varepsilon_{01}(H_2)$ & False-positive corruption rate (eq.~(19)) \\[6pt]

\multicolumn{2}{l}{\textit{Nested outcome regressions}} \\[3pt]
$\mu_2(H_2)$ & Inner outcome regression (eq.~(9)) \\
$\mu_1(H_1)$ & Middle outcome regression (eq.~(10)) \\
$\mu_0(H_0)$ & Outer outcome regression (eq.~(11)) \\
$f_0(x) = \mathbb{E}[Y^* \mid X=x]$ & Conditional fraud probability (eq.~(4)) \\[6pt]

\multicolumn{2}{l}{\textit{Estimator and score}} \\[3pt]
$\varphi^{\mathrm{corr}}$ & Corruption-corrected sequential score (eq.~(22)) \\
$\hat\Psi^{\mathrm{corr}}_{\mathrm{STR}}$ & Corruption-corrected STR (eq.~(23)) \\
$\boldsymbol{\eta} = (e,r,p,\mu_0,\mu_1,\mu_2)$ & Nuisance function collection \\[6pt]

\multicolumn{2}{l}{\textit{Efficiency and concentration}} \\[3pt]
$\sigma^2_{\mathrm{eff}}$ & Semiparametric efficiency bound (eq.~(26)) \\
$\hat\sigma^2_{\mathrm{STR}}$ & Plug-in variance estimator (eq.~(29)) \\
$B = 1/(\underline{e}\,\underline{r}\,\underline{p}\,(1-\bar\varepsilon_{10}-\bar\varepsilon_{01}))$ & Maximum inverse-propensity weight (Thm.~35) \\
$n^*(\epsilon,\alpha)$ & Critical sample size (Cor.~36) \\[6pt]

\multicolumn{2}{l}{\textit{Shrinkage}} \\[3pt]
$\lambda_i$ & EB shrinkage weight for issuer $i$ (eq.~(34)) \\
$\hat r^{\mathrm{EB}}_i$ & Shrinkage-regularized reporting propensity (eq.~(33)) \\
$\hat p^{\mathrm{EB}}_i$ & Shrinkage-regularized delay propensity (eq.~(35)) \\
$\hat e^{\mathrm{EB}}_i$ & Shrinkage-regularized authorization propensity (eq.~(36)) \\
$\sigma^2_B$ & Between-issuer variance (eq.~(34)) \\[6pt]

\multicolumn{2}{l}{\textit{Pseudo-labels and downstream model}} \\[3pt]
$U_t$ & STR pseudo-outcome for transaction $t$ (eq.~(38)) \\
$\hat y_t$ & Corrected pseudo-label (eq.~(40)) \\
$\mathcal{G}$ & Function class for downstream regression (eq.~(39)) \\[6pt]

\multicolumn{2}{l}{\textit{Sensitivity analysis}} \\[3pt]
$\Gamma_A$ & Sensitivity parameter for authorization ignorability (\S11.2) \\
$\Gamma_R$ & Sensitivity parameter for reporting ignorability (\S11.2) \\[6pt]

\multicolumn{2}{l}{\textit{Network-level parameters (Def.~45)}} \\[3pt]
$\bar e = \mathbb{E}[e_0(H_0)]$ & Average authorization rate \\
$\bar r = \mathbb{E}[r_0(H_1)]$ & Average reporting rate \\
$\bar p(\Delta) = \mathbb{E}[p_0(H_2, \Delta)]$ & Average maturity rate at delay $\Delta$ (eq.~(41)) \\
$\bar q(\Delta) = \bar e\,\bar r\,\bar p(\Delta)$ & Average observation rate \\
$\gamma = (1-\varepsilon_{10}-\varepsilon_{01})^2$ & Corruption penalty \\
$\eta = (1+\mathrm{CV}^2_q)(1+\rho_{f,q}\,\mathrm{CV}_f\,\mathrm{CV}_{1/q})$ & Heterogeneity penalty \\
$\mathrm{CV}^2_q \approx \mathrm{CV}^2_e + \mathrm{CV}^2_r + \mathrm{CV}^2_p$ & Squared coefficient of variation of observation propensity \\
$\rho_{f,q}$ & Correlation between $f_0$ and $1/q_0$ \\[6pt]

\multicolumn{2}{l}{\textit{Optimal training delay (\S10.6)}} \\[3pt]
$\bar p(\delta) = \bar p_\infty(1 - e^{-(\lambda\delta)^\beta})$ & Maturity curve, Weibull parameterization (Def.~43, eq.~(42)) \\
$\lambda$ & Label arrival rate parameter (eq.~(42)) \\
$\beta$ & Maturity curve shape parameter (eq.~(42)) \\
$\bar p_\infty$ & Asymptotic maturity rate (Def.~43) \\
$\nu$ & Fraud distribution drift rate (Def.~44, eq.~(43)) \\
$f_t(x) = \Pr(Y^*=1 \mid X=x, \text{time}=t)$ & Time-varying fraud probability (Def.~44) \\
$\mathcal{E}_{\mathrm{stat}}(\Delta)$ & Statistical error at maturity delay $\Delta$ (eq.~(45)) \\
$\mathcal{E}_{\mathrm{drift}}(\Delta)$ & Staleness error at maturity delay $\Delta$ (eq.~(46)) \\
$\mathcal{E}_{\mathrm{irr}}$ & Irreducible (Bayes) error \\
$\mathcal{E}_{\mathrm{total}}(\Delta)$ & Total prediction error (eq.~(44)) \\
$C_1 = \pi(1-\pi)\eta\,/\,(n\bar e\,\bar r\,\gamma)$ & Label-quality coefficient (eq.~(47)) \\
$C_2 = \nu$ & Drift cost coefficient (eq.~(47)) \\
$\Delta^*_{\mathrm{STR}}$ & Optimal training delay under STR (Thm.~46, eq.~(57)) \\
$\Delta^*_{\mathrm{naive}}$ & Optimal training delay under na\"ive estimator \\
$B(\Delta)$ & Selection bias of na\"ive estimator at maturity $\Delta$ (eq.~(55)) \\
$\zeta = \mathbb{E}[Y^* \mid O=0] - \mathbb{E}[Y^* \mid O=1]$ & Selection contrast (eq.~(55)) \\
$\epsilon_B$ & Tolerable bias threshold (eq.~(60)) \\
$V_{\mathrm{STR}} = V_{\mathrm{accuracy}} + V_{\mathrm{freshness}}$ & Investment value decomposition of STR (Cor.~52, eq.~(61)) \\[6pt]

\multicolumn{2}{l}{\textit{Target estimands}} \\[3pt]
$\Psi = \mathbb{E}[Y^*]$ & Population fraud rate (eq.~(3)) \\
$\pi = \Psi$ & Population fraud rate (used in \S10.6) \\
\end{longtable}
\end{center}

\end{document}